\documentclass[12pt]{article}
\usepackage{amsmath}
\usepackage{times}
\usepackage{graphicx}
\usepackage{color}
\usepackage{multirow}
\usepackage[authoryear]{natbib}
\usepackage{hyperref}
\usepackage{rotating}
\usepackage{bbm}
\usepackage{latexsym}

\usepackage{color}
\usepackage{amssymb}

\usepackage[dvipsnames]{xcolor}
\usepackage{amsfonts}
\usepackage{scalefnt}

\newtheorem{dfn}{Definition}
\newtheorem{thm}{Theorem}
\newtheorem{lem}{Lemma}
\newtheorem{rmk}{Remark}

\newtheorem{prop}{Proposition}

\usepackage[utf8]{inputenc}
\usepackage{caption}
\usepackage{amsmath,algorithm,tabularx}
\usepackage[noend]{algpseudocode}

\allowdisplaybreaks

\makeatletter
\newcommand{\multiline}[1]{%
  \begin{tabularx}{\dimexpr\linewidth-\ALG@thistlm}[t]{@{}X@{}}
    #1
  \end{tabularx}
}
\makeatother

\newcommand{\ctext}[1]{\hbox{\textcircled{\scriptsize{\lower0.2ex\hbox{#1}}}}}

\usepackage{latexsym}
\def\qed{\hfill $\Box$}
\usepackage{colortbl}
\usepackage{graphicx}
\usepackage{import}

\usepackage{lscape}

\usepackage{booktabs}

\usepackage{multirow}

\newenvironment{proof}{\noindent{\it Proof.~~}}{\medskip}

\usepackage[symbol]{footmisc}

\newcommand{\captionfonts}{\normalsize}

\makeatletter  
\long\def\@makecaption#1#2{%
  \vskip\abovecaptionskip
  \sbox\@tempboxa{{\captionfonts #1: #2}}%
  \ifdim \wd\@tempboxa >\hsize
    {\captionfonts #1: #2\par}
  \else
    \hbox to\hsize{\hfil\box\@tempboxa\hfil}%
  \fi
  \vskip\belowcaptionskip}
\makeatother   

\renewcommand{\thefootnote}{\normalsize \arabic{footnote}} 	

\begin{document}
\hspace{13.9cm}1

\ \vspace{20mm}\\

{\LARGE Deep Clustering with a Constraint for Topological Invariance based on Symmetric InfoNCE}

\ \\
{\bf \large 
Yuhui Zhang$^{\displaystyle 1, \displaystyle \ast}$, 
Yuichiro Wada$^{\displaystyle 2, 3, \displaystyle \ast}$, 
Hiroki Waida$^{\displaystyle 1, \displaystyle \ast}$\\
Kaito Goto$^{\displaystyle 1, \displaystyle \dag}$, 
Yusaku Hino$^{\displaystyle 1, \displaystyle \dag}$, 
Takafumi Kanamori$^{\displaystyle 1, 3, \displaystyle \ddag}$}\\
{$^{\displaystyle 1}$Tokyo Institute of Technology, Tokyo, Japan}\\
{$^{\displaystyle 2}$Fujitsu Limited, Kanagawa, Japan}\\
{$^{\displaystyle 3}$RIKEN AIP, Tokyo, Japan}

\renewcommand{\thefootnote}{\alph{footnote}}
\footnotetext{$^\ast$Equally Contributions}
\footnotetext{$^\dag$Part of this work was done while K.~Goto and Y.~Hino were  master course students at Tokyo Institute of Technology.}
\footnotetext{$^\ddag$Corresponding author}

\renewcommand{\thefootnote}{\normalsize \arabic{footnote}}


\thispagestyle{empty}
\markboth{}{NC instructions}
\ \vspace{-0mm}\\
\begin{center} {\bf Abstract} \end{center}
We consider the scenario of deep clustering, in which the available prior knowledge is limited.
In this scenario, few existing state-of-the-art deep clustering methods can perform well for both non-complex topology and complex
topology datasets.
To address the problem,  
we propose a constraint utilizing symmetric InfoNCE, which
helps an objective of deep clustering method in the scenario train the
model so as to be efficient for not only non-complex topology but also complex topology datasets.
Additionally, we provide several theoretical explanations of the reason 
why the constraint can enhances performance of deep clustering methods. 
To confirm the effectiveness of the proposed constraint, we introduce a deep clustering method named MIST, which is a combination of an existing deep clustering method and our constraint. 
Our numerical experiments via MIST demonstrate that the constraint is effective. In addition, MIST outperforms other state-of-the-art deep clustering methods for most of the commonly used ten benchmark datasets.

\section{Introduction}
\label{sec:introduction}
\subsection{Background}
\label{subsec:background}
Clustering is one of the most popular and oldest research fields of machine learning. 
Given unlabeled data points, the goal of clustering is to group them according to some criterion. 
In addition, in most of the cases clustering is performed with unlabeled datasets.
Until today, many clustering algorithms have been proposed~\citep{macqueen1967some,day1969estimating,girolami2002mercer,wang2003kernel,ng2001spectral,ester1996density,sibson1973slink}.
For a given unlabeled  dataset, the number of clusters, and a distance metric, 
K-means~\citep{macqueen1967some} 
aims to partition the dataset into the given number clusters. Especially when the squared Euclidean distance is employed as the metric, it returns convex-shaped clusters within short running time.
GMMC (Gaussian Mixture Model Clustering)~\citep{day1969estimating} assigns labels to estimated clusters after fitting GMM to an unlabeled dataset, where the number of clusters can be automatically determined by Bayesian Information Criterion~\citep{schwarz1978estimating}. 
The kernel K-means~\citep{girolami2002mercer}, kernel GMMC~\citep{wang2003kernel} and SC (Spectral Clustering)~\citep{ng2001spectral} can deal with more complicated shapes compared to K-means and GMMC.
On the other hand, 
as examples categorized into a clustering method that does not require the number of clusters, DBSCAN~\citep{ester1996density} and Hierarchical Clustering~\citep{sibson1973slink} are listed.

Although the above-mentioned classical clustering methods are useful for low-dimensional small datasets, they often fail to  
handle large and high-dimensional datasets (e.g., image/text datasets).  
Due to the development of deep learning techniques for DNNs (Deep Neural Networks), 
we can now handle large datasets with high-dimension~\citep{lecun2015deep}. 
Note that a clustering method using DNNs is refereed to as Deep Clustering method.

\subsection{Our Scenario with Deep Clustering}
\label{subsec:our scenario}
The most popular scenario for deep clustering is the domain-specific scenario (\textsf{Scenario1}) \citep{mukherjee2019clustergan,ji2019invariant,asano2019self,van2020scan,NEURIPS2020_70feb62b,NEURIPS2020_6740526b,monnier2020deep,li2020contrastive,Dang_2021_CVPR}. In this scenario, an unlabeled dataset of a specific domain and its number of clusters are given, while specific rich knowledge in the domain can be available. The dataset is often represented by raw data.
An example of the specific knowledge in the image domain is efficient domain-specific data-augmentation techniques~\citep{ji2019invariant}.
It is also known that CNN (Convolutional Neural Network)~\citep{lecun1989backpropagation} can be an efficient DNN to extract useful features from raw image data.   
In this scenario, most of the authors have proposed \textit{end-to-end} methods, while some have done
\textit{sequential} methods.
In the category of the end-to-end methods, a model is defined by an efficient DNN for a specific domain, where the input and output of the DNN are (raw) data and its predicted cluster label, respectively.
The model is trained under a particular criterion as utilizing the domain-specific knowledge.  
In the category of the sequential methods, the clustering DNN is often constructed in the following three steps: 
1) Create a DNN model that extracts features from data in a specific domain, followed by an MLP (Multi-Layer Perceptron) predicting cluster labels.
2) Train the feature-extracting DNN using an unlabeled dataset and domain-specific knowledge. Then, freeze the set of trainable parameters in the feature extracting DNN. 
3) 
Train the MLP with the features obtained from the feature extracting DNN, and domain-specific knowledge.

The secondly important scenario is called the non-domain-specific  scenario (\textsf{Scenario2}) \citep{springenberg2015unsupervised,xie2016unsupervised,jiang2016variational,ijcai2017-243,hu2017,shaham2018,nazi2019generalized,gupta2019unsupervised,9413131}.
In this scenario, an unlabeled dataset and the number of clusters are given, while a few generic assumptions for a dataset can be available, such as 1) the cluster assumption, 2) the manifold assumption, and 3) the smoothness assumption; see~\citet{books/mit/06/CSZ2006} for details. In addition, unlabeled data is often represented by a feature vector.
As well as \textsf{Scenario1},
many authors have proposed end-to-end methods where an MLP model is trained
by utilizing some generic assumption, while some have done sequential methods.

A scenario apart from \textsf{Scenario1} and \textsf{Scenario2} is reviewed in Appendix~\ref{subappend:Review of Deep Clustering Methods without Number of Clusters}.

In this study, we focus on \textsf{Scenario2} due to the following two reasons.
The first reason is that we do not always encounter \textsf{Scenario1} in practice.
The second one is that if we prepare an efficient deep clustering algorithm of \textsf{Scenario2}, this algorithm can be incorporated into the third step of the sequential method of \textsf{Scenario1}.

\subsection{Motivation behind and Goal}
\label{subsec:motivation}

To understand the pros and cons
of the recent state-of-the-art methods under \textsf{Scenario2}, we conduct preliminary experiments with eight deep clustering methods shown in Table~\ref{tab:types of deep clustering methods}. Among the eight methods, SCAN, IIC, CatGAN, and SELA are originally proposed in \textsf{Scenario1}. Therefore, we redefined the four methods to keep fairness in comparison; see Appendix~\ref{append:Implementation Details with Previous Methods} for details of the alternative definitions that fit to our setting.
For the datasets, we employed ten datasets as shown in the first row of Table~\ref{tb:results of clustering accuracy}. Here, Two-Moons and Two-Rings are synthetic, and the remaining eight are real-world; see more details in Section~\ref{subsec:dataset description and evaluation metric} and Appendix~\ref{append:Detail of Dataset}. 
Throughout our experiments, we have found that the real-world datasets can be clustered well by performing K-means with the Euclidean distance, while the synthetic datasets cannot. We therefore regard the synthetic (resp. real-world) datasets as having  \textit{complex topology} (resp. \textit{non-complex topology}).
Intuitively, a dataset is topologically non-complex if it is 
clustered well by K-means with the Euclidean distance.
Otherwise, the dataset  is thought to have a complex topology. For more details of complex and non-complex topologies, see Appendix~\ref{append: Complex and Non-Complex Topology Datasets}.

\newcommand{\STAB}[1]{\begin{tabular}{@{}c@{}}#1\end{tabular}}

\begin{table}[!t]
    \centering
    \caption{
    Representative deep clustering methods (either sequential or end-to-end). They are categorized into six types; $\mathfrak{T}_{1}$ to $\mathfrak{T}_{6}$. The example methods are shown for each type. Here, Seq. is abbreviation of the word "Sequential".
    }
    \scalebox{0.85}{
    \begin{tabular}{cccc}
         & Type & Explanation & Example(s) \\
         \hline
         \multirow{2}{*}{\STAB{\rotatebox[origin=c]{90}{Seq.}}} & $\mathfrak{T}_{1}$ & Plain embedding based & DEC~\citep{xie2016unsupervised}, SCAN~\citep{van2020scan} \\
         & $\mathfrak{T}_{2}$ & Spectral embedding based &  SpectralNet~\citep{shaham2018} \\
         \hline
         \multirow{4}{*}{\STAB{\rotatebox[origin=c]{90}{End-to-End}}}& $\mathfrak{T}_{3}$ & Variational bound based &  VaDE~\citep{jiang2016variational}\\
         & $\mathfrak{T}_{4}$ & Mutual information based & IMSAT~\citep{hu2017}, IIC~\citep{ji2019invariant}\\
         & $\mathfrak{T}_{5}$ & Generative adversarial net based  & CatGAN~\citep{springenberg2015unsupervised}\\
         & $\mathfrak{T}_{6}$ & Optimal transport based  & SELA~\citep{asano2019self} \\
         \hline
    \end{tabular}
    }
    \label{tab:types of deep clustering methods}
\end{table}

\begin{table*}[!t]
  \centering
   \caption{
   Comparison of classical and deep clustering methods in terms of clustering accuracy (\%
   One and seven trials are respectively conducted for the classical (top 3 methods) and the deep clustering methods, respectively. Mean and standard deviation of their accuracy are reported.
   Symbol "-" means that result was not returned by the clustering algorithm within one hour of running.
   Numbers with $\dagger$  are copied from the corresponding studies. 
   The symbol $\S$ means that an original method is redefined for \textsf{Scenario2}.}
   \label{tb:results of clustering accuracy}
   \scalebox{0.65}{
      \begin{tabular}{cccccccccccc} 
      \toprule
                &    & Two-Moons& Two-Rings  & MNIST        & STL              & CIFAR10          & CIFAR100         & Omniglot        & 20news  & SVHN  & Reuters10K  \\ \midrule%
        \multirow{3}{*}{\STAB{\rotatebox[origin=c]{90}{\footnotesize Classical Methods}}} & K-means     & 75.1     & 50.0        & 53.2   & 85.6& 34.4    & 21.5     & 12.0     & 28.5  & 17.9  &  54.1       \\ %
        & SC          & \textbf{100.0}    & \textbf{100.0}       & 63.7   & 83.1& 36.6    & -        & -        & -    & 27.0  &  43.5          \\ %
        & GMMC        & 85.9     & 50.3        & 37.7   & 83.5& 36.7    & 22.5     & 7.6      &\textbf{39.0}   & 14.2 &  67.7           \\ \midrule%
        \multirow{8}{*}{\STAB{\rotatebox[origin=c]{90}{\footnotesize Representative Existing Deep Methods}}} & DEC         & 70.3(7.1) & 50.7(0.3) &$\dagger$84.3&$\dagger$78.1(0.1)&$\dagger$46.9(0.9)&$\dagger$14.3(0.6)&$\dagger$5.7(0.3)&30.1(2.8)&$\dagger$11.9(0.4) &$\dagger$67.3(0.2) \\ %
        & SpectralNet &\textbf{100(0)} & 99.9(0.0) &$\dagger$82.6(3.0)&90.4(2.1) & 44.3(0.6) & 22.7(0.3) & 2.5(0.1) & 6.3(0.1)  &10.4(0.1) &$\dagger$66.1(1.7)    \\ %
        & VaDE       &50.0(0.0) &50.0(0.0) & 83.0(2.6) & 68.8(12.7) & 39.5(0.7) & 12.13(0.2) & 1.0(0.0) & 12.7(5.1) & 32.9(3.2)  & 70.5(2.5) \\ %
        & IMSAT      &86.3(14.8)&71.3(20.4) &98.4(0.4)&93.8(0.5)&45.0(0.5)&27.2(0.4)&\textbf{24.6(0.7)}&37.4(1.4)& 54.8(5.1) &72.7(4.6) \\ %
        & IIC$\S$        &77.2(18.4)&66.2(21.5)&45.4(8.3)&39.0(8.7)&23.9(4.8)&4.4(0.9)&2.3(0.4)& 14.9(5.3)& 17.1(1.1)& 58.3(2.2)     \\ %
        & CatGAN$\S$      &81.6(5.3) &53.7(2.5) & 15.2(3.5) & 32.9(3.0) & 15.1(2.4) & 5.1(0.5) & 3.3(0.2) & 19.5(6.5) & 20.4(0.9) & 43.6(7.3)         \\ %
        & SELA$\S$      & 62.7(9.5) & 52.6(0.1) & 46.1(3.4) &68.6(0.4) & 29.7(0.1) & 18.8(0.3) & 11.3(0.2) & 20.2(0.1) & 19.3(0.3) & 49.1(1.7)\\ %
        & SCAN$\S$       & 85.7(22.6) & 75.1(23.1) & 82.1(3.7) & 92.8(0.5) & 43.3(0.6) & 24.6(0.1) & 17.6(0.4) & 38.4(1.1) & 23.2(1.6) & 63.4(4.2)\\ \midrule%
        & MIST via $\hat{I}_{\rm nce}$ &  \textbf{100(0)}  & 95.2(2.0)  & 98.0(1.0)  & 94.2(0.4) & 48.9(0.7) & \textbf{27.8(0.5)} & 24.0(1.0) & 38.3(2.8) & 58.7(3.5) & 72.7(3.3) \\
        & {\bf MIST} (ours) & \textbf{100(0)}  & 93.3(16.3) &\textbf{98.6(0.1)}&\textbf{94.5(0.1)}&\textbf{49.8(2.1)}&\textbf{27.8(0.5)}&\textbf{24.6(1.1)}&38.8(2.3)&\textbf{60.4(4.2)}&\textbf{73.4(4.6)} \\ %
      \bottomrule
      \end{tabular}
                  }
\end{table*}

The experimental results are shown in
Table~\ref{tb:results of clustering accuracy}.
As shown in the table, 
for the complex topology datasets, all the eight deep clustering methods of Table~\ref{tab:types of deep clustering methods} except for SpectralNet fail to cluster data points, compared to SC of a classical method. 
Note that, among the eight compared deep clustering methods, ones that incorporate the K-NN  (K-Nearest Neighbor)  graph tend to perform better for the complex topology datasets than ones that do not. Here, the methods incorporating the graph are SpectralNet, IMSAT, IIC, and SCAN.
On the other hand, for the non-complex topology datasets, only IMSAT sufficiently outperforms the three classical methods on average. 
In addition, the average clustering performance of IMSAT over the ten datasets is the best among the eight methods.
As Table~\ref{tb:results of clustering accuracy} suggests, to the best of our knowledge, almost none of previous deep clustering methods sufficiently perform well for both non-complex and complex datasets. Those results motivate our study.

Our aim is to propose a constraint that helps an objective of deep clustering method in \textsf{Scenario2}  train the model so as to be efficient for not only non-complex topology but also complex topology datasets. Such a versatile deep clustering objective can be helpful for users.

\subsection{Contributions}
\label{subsec:constributions}
To achieve our goal, we propose the constraint for topological invariance. 
For two data points close to each other, 
the corresponding class probabilities computed by an MLP should be close. 
For example, in b) of Figure~\ref{fig:weakness of imsat and spectralnet}, any pair of two red points is  close, while any pair of a red and a blue point are apart from each other. This constraint is introduced as a  regularization based on the maximization of  {\it symmetric InfoNCE} between the two probability vectors; see Section~\ref{subsec:symmetric infonce}. In order to define the two probability vectors,
    we introduce two kinds of \emph{paired data}. 
    One is used for a non-complex topology dataset, which is based on a K-NN graph with the Euclidean metric; see Definition~\ref{def:process T_e}.  The other is used for a complex topology dataset, and it is based on a K-NN graph with the geodesic metric on the graph; see Definition~\ref{def:process T_g}. 
    Both graphs are defined only with an unlabeled dataset. 
    The geodesic metric is defined by the graph-shortest-path distance on the K-NN graph constructed with the Euclidean distance. 

We emphasize that under \textsf{Scenario2}, it is impossible to incorporate powerful domain-specific knowledge into a deep clustering method. In addition,  the maximization of the symmetric InfoNCE has not been studied yet in the context of deep clustering. 
Moreover, we present the legitimacy of our topological invariant constraint by showing several theoretical findings from mainly two perspectives: 1) in Section~\ref{subsubsubsec: analysis from viewpoints of mi}, the motivations and the potential of the proposed constraint are clarified from the standpoint of statistical dependency measured by MI (Mutual Information), and 2) in Section~\ref{subsubsubsec:futher motivations}, an extended result of the theory on contrastive representation learning derived from \citet{wang2022chaos} is presented to discuss 
advantages of the symmetric InfoNCE over InfoNCE for deep clustering.
Note that InfoNCE~\citep{oord2018representation} was ininitially used for domain-specific representation learning such as vision tasks, NLP, and reinforcement learning.  
Here, a purpose of representation learning is to extract useful features that can be applied to a variety range of machine learning methods~\citep{bengio2013representation}. 
Whereas, the purpose of clustering is to annotate the cluster labels to unlabeled data points.

\begin{figure}[!t]
\centering
\def\svgwidth{\linewidth}
\includegraphics[width=\linewidth]{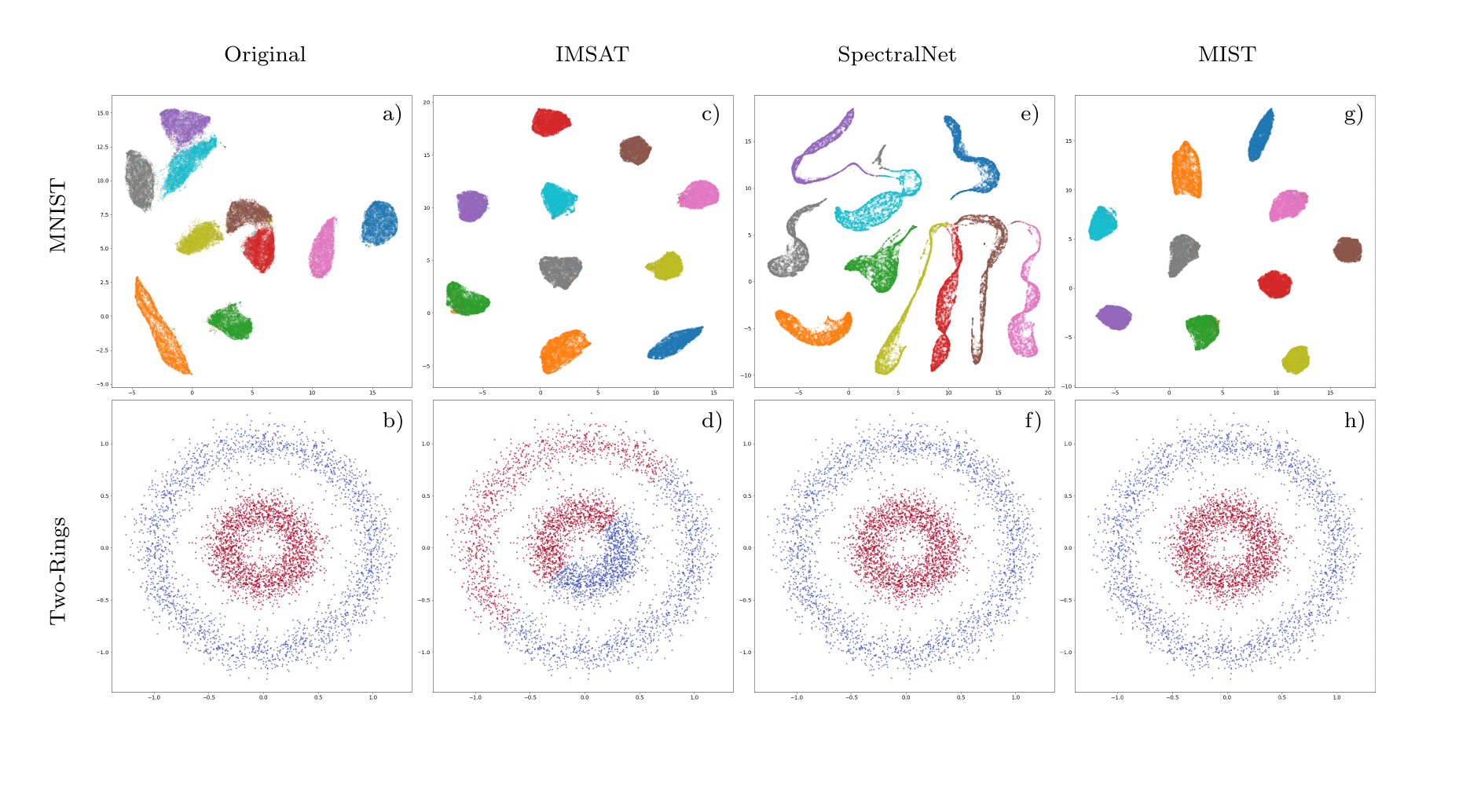}
\caption{
Two-dimensional visualizations of clustering results.  
In the first row except for a), three visualizations are obtained via the following procedure: using the trained MLP, compute the output for each feature in MNIST, where dimension of the feature is 784. Then, the outputs are transformed into two-dimensional vectors by UMAP. Thereafter, true labels are assigned to those vectors. As for a), the original features are directly transformed into two-dimensional vectors by UMAP, and then the labels are assigned to the transformed vectors. For the second row, the true labels (resp. predicted cluster labels by the trained MLP) are assigned to the original features for obtaining b) (resp. d), f), and h)). Note that the original features of Two-Rings belong to $\mathbb{R}^2$.
}\label{fig:weakness of imsat and spectralnet}
\end{figure}

The main contributions are summarized as follows:
\begin{enumerate}
    \item[1)] 
    We propose a topological invariant constraint via the symmetric InfoNCE for the purpose of deep clustering in \textsf{Scenario2}, and then show the advantage by providing analysis from several theoretical aspects.

    \item[2)] 
    To evaluate the proposed constraint in numerical experiments, by applying the  constraint to IMSAT, we define a deep clustering method named {\it MIST
    (Mutual Information maximization with local Smoothness and Topologically invariant constraints)}. In the experiments, we confirm that the proposed constraint enhances the accuracy of a deep clustering method.
    Furthermore, to the best of our knowledge,  MIST achieves state‐of‐the‐art clustering accuracy in \textsf{Scenario2} for not only non-complex topology datasets but also complex topology datasets. 
    
\end{enumerate}
In Figure~\ref{fig:weakness of imsat and spectralnet}, a positive impact of the topological invariant constraint toward IMSAT is visualized via UMAP~\citep{mcinnes2018umap}; compare d) and h) in the figure. See further details of Figure~\ref{fig:weakness of imsat and spectralnet} and more two-dimensional visualizations of MIST for other datasets in Appendix~\ref{append:Visualization Details in Fig of imsat and spectralnet weakness}.

This paper is organized as follows. 
In Section~\ref{sec:related-works}, we overview related works.
In Section~\ref{sec:proposed-constraint}, 
we explain details of the topological invariant constraint, and then show the theoretical properties. 
In numerical experiments of Section~\ref{sec:numerical-experiments}, we define MIST. Then, we evaluate the proposed constraint via MIST using two synthetic datasets and eight real-world datasets. 
In the same section, some case studies are also provided. 
In Section~\ref{sec:conclusion-and-future-works}, we conclude this paper.

\section{Related Works}
\label{sec:related-works}

In Section~\ref{subsec: representatitve deep methods in scenario2}, we briefly explain representative deep clustering methods shown as examples in Table~\ref{tab:types of deep clustering methods} since they are compared methods in numerical experiments of Section~\ref{sec:numerical-experiments}. Then, details of InfoNCE, which is closely related to our topological invariant constraint, is introduced in Section~\ref{subsec:infonce}. 
\subsection{Representative Deep Clustering Methods}
\label{subsec: representatitve deep methods in scenario2}

Let us start from  sequential methods of $\mathfrak{T}_1$ and $\mathfrak{T}_2$ in Table~\ref{tab:types of deep clustering methods}. In DEC~\citep{xie2016unsupervised} of $\mathfrak{T}_1$, at first, a stacked denoising Auto-Encoder (AE) is trained with a set of unlabeled data points to extract the feature. Using the trained encoder, we can have the feature vectors. Then, K-means is used on the vectors in order to obtain the set of centroids. After that, being assisted by the centroids, the encoder is refined for the clustering. In SCAN~\citep{van2020scan} of $\mathfrak{T}_1$, a ResNet~\citep{he2016deep} is trained using augmented raw image datasets under SimCLR~\citep{chen2020simple} criterion to extract the features. Then, the clustering MLP added to the trained ResNet is tuned by maximizing Shannon-entropy of the cluster label while two data points in a nearest neighbor relationship are forced to have same cluster label. Then, in SpectralNet~\citep{shaham2018} of $\mathfrak{T}_2$, at first, a Siamese network is trained by the predefined similarity scores on the K-NN  graph. Then, being assisted by the trained Siamese network, a clustering DNN is trained. Note that the two networks (i.e., Siamese net and clustering net) are categorized into this method. 

With regard to end-to-end methods of $\mathfrak{T}_3$ to $\mathfrak{T}_6$, in VaDE~\citep{jiang2016variational} of $\mathfrak{T}_3$, a variational AE is trained so that the latent representation of unlabeled data points has the Gaussian mixture distribution. Here, the number of mixture components is equal to the number of clusters.
For IMSAT and IIC of $\mathfrak{T}_4$, in IMSAT~\citep{hu2017}, the clustering model is trained via maximization of the MI between a data point and the cluster label, while regularizing the model to be locally smooth; see Appendix~\ref{subsec:imsat}. 
Likewise, IIC~\citep{ji2019invariant} returns the estimated cluster labels using the trained model for clustering. The training criterion is based on maximization of the MI between the cluster label of a raw image and the cluster label of the transformed raw image; see Appendix~\ref{subsec:Invariant Information Clustering}. IIC employs a CNN-based clustering model to take advantages of image-specific prior knowledge. Furthermore, in CatGAN~\citep{springenberg2015unsupervised} of $\mathfrak{T}_5$, the neural network for clustering is trained to be robust against noisy data. Here, the noisy data is defined as a set of fake data points obtained from the generator that is trained to mimic the distribution of original data. 
Lastly, in SELA~\citep{asano2019self} of $\mathfrak{T}_6$, a ResNet is trained for clustering using an augmented unlabeled dataset with pseudo labels under the cross-entropy minimization criterion. The pseudo labels are updated at the end of every epoch by solving an optimal transporting problem.

\subsection{Info Noise Contrastive Estimation }
\label{subsec:infonce}
In representation learning, InfoNCE (Info Noise Contrastive Estimation) based on NCE has recently become a popular objective. 
The ($q$-)InfoNCE of the random variables $Z$ and $Z'$ is defined by 
\begin{align}
\label{eq:infonce}
 I_{\mathrm{nce},q}(Z;Z') = \mathbb{E}[q(Z,Z')]-\mathbb{E}_Z\big[ \log\mathbb{E}_{Z'}[e^{q(Z,Z')}] \big], 
\end{align}
where $q$ is called \textit{critic function} that quantifies the dependency between $Z$ and $Z'$. 
For any critic, $q$-InfoNCE provides a lower bound of an MI. 
Furthermore, we can see that the maximum value of $I_{\mathrm{nce},q}$ is the MI, which is attained by $q(z,z')= \log p(z|z') + [\text{any function of $z$}]$;
see~\citet{pmlr-v97-poole19a,pmlr-v80-belghazi18a} and Eq.\eqref{eq:optimal critic for infonce} for details. 
Here, $p(z|z')$ is the conditional probability of $z$ given $z'$.
When it comes to the image processing~\citep{chen2020simple,grill2020bootstrap}, 
the observations, $z$ and $z'$, are often given as different views or augmentations of an image. 
For example, $z$ and $z'$ are observed by rotating, cropping, or saturating the same source image. 
Such a pair of images are regarded as positive samples (pair).
A pair of transformed images coming from different source images are negative samples (pair).

Suppose we have samples $z_1,\cdots,z_m$ and $z_1',\cdots,z_m'$,  such that $z_i$ and $z_i'$ are all positive samples for $i=1,\cdots,m$
and $z_i$ and $z_j'$ for $i\neq j$ are negative samples. 
Then, InfoNCE is empirically approximated by 
\begin{align}
\label{eq:estimated infonce}
 \hat{I}_{\text{nce},q}
 =
 \frac{1}{m}\sum_{i=1}^{m}
 \log\frac{e^{q\left(z_i,z'_{i}\right)}}{\frac{1}{m}\sum_{j=1}^{m}e^{q\left(z_i,z_{j}'\right)}}. 
\end{align}
In order to approximate the MI by InfoNCE, one can use a parameterized model with a critic function $q$. 
In the original work of InfoNCE~\citep{oord2018representation} 
the critic $q_W(z,z')=z^{T}Wz'$ with the weight matrix $W$ is employed.
Then, the maximum value of $\hat{I}_{\text{nce},q_W}$ w.r.t. $W$ is computed to estimate the MI. 
As pointed out by~\citet{oord2018representation,pmlr-v97-poole19a,tschannen2019mutual},
The empirical InfoNCE is bounded above by $\log m$, making the bound loose when $m$ is small, or the MI $I(Z;Z')$ is large.

\section{Proposed Constraint and its Theoretical Analysis}
\label{sec:proposed-constraint}

\subsection{Notations}
\label{subsec:Notations}
In the following, $\mathcal{D} = \{x_i\}_{i=1}^n$ ($\forall i; x_i \in \mathbb{R}^d$)
is a set of unlabeled data, where 
$n$ is the number of data points and $d$ is the dimension of a data point. 
The number of clusters is denoted by $C$. Here, let $y_i$ denote the true label of $x_i$.
Let us define $p\left(y|x\right)$ as the conditional discrete probability of a cluster label $y\in\{1,2,...,C\}$ for a data point $x\in\mathbb{R}^d$.
The random variable corresponding to $x$ (resp. $y$) is denoted by $X$ (resp. $Y$). 
Let $\Delta^{C-1} = \{z \in \mathbb{R}^{C}\;|\; z \geq 0, z^\top \mathbf{1}=1\}$ be the $(C-1)$-dimensional probability simplex,
where $\mathbf{1}$ is the $C$-dimensional vector $(1,1,...,1)^\top$.

\begin{dfn}[MLP model $g_{\theta}$]
\label{def:g_theta}
Consider a DNN model $g_{\theta}(x): \mathbb{R}^{d} \to \Delta^{C-1}$ with trainable set of parameters $\theta$, %
where the activation for the last layer is defined by the $C$-dimensional softmax function. The $y$-th element of $g_{\theta}(x)$ is denoted by $g_{\theta}^y(x)$. %
Let $\theta^\ast$ denote the trained set of parameters via a clustering objective, using an unlabeled dataset $\mathcal{D}$.
The predicted cluster label of $x_i \in \mathcal{D}$ is defined by $\hat{y}_i = {\rm argmax}_{y\in\{1,\cdots,C\}}g_{\theta^\ast}^y(x_i)$. 
\end{dfn}

\subsection{Preliminary}
\label{subsec:Preliminary}
Consider \textsf{Scenario2} of Section~\ref{subsec:our scenario}, where a set of unlabeled data 
$\mathcal{D} = \{x_i\}_{i=1}^n, x_i \in \mathbb{R}^d$ and the number of clusters $C$ are given, while a few
generic assumptions for the dataset can be available. We firstly in Section~\ref{subsec:symmetric infonce} introduce the topological invariant constraint based on symmetric InfoNCE and an MLP model $g_\theta$. %
Then, in Section~\ref{subsubsec:theoretical properties}, some relations between the symmetric InfoNCE and the corresponding MI are theoretically analyzed. Thereafter, based on the analysis, we explain theoretical advantages of the symmetric InfoNCE over existing popular constraints such as IIC and InfoNCE in terms of deep clustering.

Before stating the mathematical definitions and the properties, we briefly explain why the symmetric InfoNCE can enhance a deep clustering method as a topological invariant constraint.
As mentioned in  Section~\ref{subsec:constributions}, 
the topological invariant constraint is expected to regularize $g_\theta$
so as to be
$g_\theta (X) \approx g_\theta (X^\prime) \in \Delta^{C-1}$ for any geodesically-close two data points $X, X^\prime \in \mathcal{D}$ in the original space $\mathbb{R}^d$.
In other words,  predicted cluster labels of $X$ and $X^\prime$ are enforced to be same.
For the regularization,
InfoNCE and its variants %
are potentially useful. The reason is that in representation learning InfoNCE is empirically successful for making the following two feature vectors close to each other: 1) a feature vector returned by a DNN with a raw data as an input, 2) a feature returned by the same DNN with an augmented data from the raw
data~\citep{oord2018representation,chen2020simple}. Note that feature vectors described in 1) and 2) are  not in $\Delta^{C-1}$
but usually in
the high-dimensional Euclidean space.
In this study,  the symmetric InfoNCE between $X$ and $X'$ is proposed as a constraint for topological invariance. 
The pair $(X, X^\prime)$ is given by $(X, T(X))$, where 
$T(X)$ is a transformation of $X$: some practical tranformations are introduced in Definition~\ref{def:process T_e} and~\ref{def:process T_g} of Section~\ref{subsec:symmetric infonce}.

\subsection{Topological Invariant Constraint}
\label{subsec:symmetric infonce}
We aim to design a constraint for topological invariance  that should satisfy the following condition; 
if clusters of $\mathcal{D}$ have a non-complex (resp. complex) topology, the constraint assists a model $g_\theta$ to predict the same cluster labels for $x \in \mathcal{D}$ and $x^\prime \in \mathcal{D}$ whenever $x$ and $x'$ are close to each other in terms of the Euclidean (resp. geodesic)  distance.
In the sequel, we define the constraint via symmetric InfoNCE. Then, we investigate its theoretical properties.

Firstly, let us define a function $q: \Delta^{C-1} \times \Delta^{C-1} \to \mathbb{R}$ as follows:
\begin{align}
\label{eqn:alpha-tau-model}
    q(z, z^\prime) = 
    \log\left(\exp_{\alpha}\big(\tau(z^\top z'-1)\big)\right), 
\end{align}
where $\alpha \in \mathbb{R}$ and $0 \leq \tau \leq |1-\alpha|^{-1}$. In addition, for $u \in \mathbb{R}$,
$\exp_{\alpha}(u)$ is defined by 
$[1+(1-\alpha) u]_{+}^{1 /(1-\alpha)}$ for $\alpha\neq 1$
and $e^u$ for $\alpha=1$, where $[\;\cdot\;]_{+} = \max\{\cdot,0\}$. The function $q$ of Eq.\eqref{eqn:alpha-tau-model} w.r.t. $z$ and $z'$ is maximized if and only if $z$ and $z'$ are the same one-hot vector. On the other hand, it is minimized if and only if $z^\top z' = 0$ (i.e., $z$ and $z'$ are orthogonal to each other).

We then define transformation $T$  for constructing a pair of geodesically-close two data points $(X, T(X))$ based on $X$ as follows. 
\begin{dfn}[Transformation $T$]
\label{def:random transformation function}
Let $X$ be a $d$-dimensional random variable. Then, $T:\mathbb{R}^d\to\mathbb{R}^d$ denote the transformation of $X$, and it is also a random variable. The realization is denoted by $t:\mathbb{R}^d \to \mathbb{R}^d$. Given a data point $x$, the function $t$ is sampled from the conditional probability $p(t|x)$. 
\end{dfn}

The probability $p(t|x)$ is defined through a generative process. In this study, two processes, $\mathcal{T}_{\mathfrak{e}}$ and $\mathcal{T}_{\mathfrak{g}}$, are considered. 
The first (resp. second) one, $\mathcal{T}_{\mathfrak{e}}$ (resp. $\mathcal{T}_{\mathfrak{g}}$), is defined using the K-NN graph with the Euclidean (resp. geodesic) distance, and is employed for non-complex (resp. complex) topology datasets.

\begin{dfn}[Generative process $\mathcal{T}_{\mathfrak{e}}$]
\label{def:process T_e}
Given an unlabeled dataset $\mathcal{D}=\{x_i\}_{i=1}^n$, a natural number $K_0$, and $\beta \in [0,1)$ as inputs, 
then the generative process of a transformation is defined as follows. 1) At the beginning, 
build a K-NN graph with $K=K_0$ on $\mathcal{D}$ based on the Euclidean distance. 2) For all $k=1+\lfloor\beta K_0\rfloor,\cdots,K_0$, 
define a function $t^{(i\to k)}:\mathbb{R}^d \to \mathbb{R}^d$ by $x^{(k)}_i = t^{(i\to k)}(x_i)$, where $x^{(k)}_i$ is a $k$-th nearest neighbor data points of $x_i$ on the graph. 3) Define the conditional distribution $p(t|x_i)$ 
as the uniform distribution on $t^{(i\to k)}, 
k=1+\lfloor\beta K_0\rfloor,\cdots,K_0$.
\end{dfn}

\begin{dfn}[Generative process $\mathcal{T}_{\mathfrak{g}}$]
\label{def:process T_g}
Given an unlabeled dataset $\mathcal{D}=\{x_i\}_{i=1}^n$, a natural number $K_0$, and $\beta \in [0,1)$ as inputs, then firstly build a K-NN graph based on the Euclidean distance with $K=K_0$ on $\mathcal{D}$. 
Then, in order to approximate the geodesic distance 
between $x_i$ and $x_j$, compute the graph-shortest-path distance. 
Let $\mathfrak{g}_{ij}$ be the approximated geodesic distance between $x_i$ and $x_j$, 
and $\mathfrak{G}$ be an $n\times n$ matrix 
$(\mathfrak{g}_{ij})_{i,j=1,\cdots,n}$. 
For each $i$, let $\mathcal{M}_{i} = \{j\;|\;0<\mathfrak{g}_{ij}<\infty\}$ be 
the set of indices whree each $x_{j}$ is a neighborhood of $x_{i}$ under the geodesic distance.
For each $i$, the generative process of the transformation $t$
is given as follows. 1) For all $k=|\mathcal{M}_{i}|-\lfloor\beta K_{\mathfrak{g}}\rfloor+1,\cdots,|\mathcal{M}_{i}|$,
define a function $t^{(i\to k)}:\mathbb{R}^d \to \mathbb{R}^d$ by $x^{(k)}_i = t^{(i\to k)}(x_i)$, 
where $x^{(k)}_i$ is a $k$-th geodesically nearest neighbor data points from $x_i$ on $\mathfrak{G}$ except in the case of $\mathfrak{g}_{ij}=\infty$ and $\mathfrak{g}_{ij}=0$.
2) Define the conditional distribution $p(t|x_i)$ 
as the uniform distribution on $t^{(i\to k)}, 
k=|\mathcal{M}_{i}|-\lfloor\beta K_{\mathfrak{g}}\rfloor+1,\cdots,|\mathcal{M}_{i}|$.

\end{dfn}
The time and memory complexities with $\mathcal{T}_{\mathfrak{e}}$ and $\mathcal{T}_{\mathfrak{g}}$ are provided in Appendix~\ref{append:time and memory complex with t_e and t_g}.
Intuitively, when $\beta=0$, $\mathcal{T}_{\mathfrak{e}}$ picks a random neighbor of $x$ as $T(x)$ in the $K_0$-nearest neighbor graph,
while $\mathcal{T}_{\mathfrak{g}}$ picks the a random neighbor by the geodesic metric induced by the $K_0$-nearest neighbor graph.
Larger $\beta$ disables $\mathcal{T}_{\mathfrak{e}}$ and $\mathcal{T}_{\mathfrak{g}}$ from picking closest neighbors.

Using the function $q$ of Eq.\eqref{eqn:alpha-tau-model}, let us define $I_{\rm nce} \equiv I_{\mathrm{nce},q}(g_{\theta}(X);g_{\theta}(T(X)))$ and $I'_{\rm nce} \equiv I_{\mathrm{nce},q}(g_{\theta}(T(X));g_{\theta}(X))$; see $I_{{\rm nce}, q}$ in Eq.\eqref{eq:infonce}. 
We then define the symmetric InfoNCE by $\left(I_{\mathrm{nce}} + I'_{\mathrm{nce}}\right) / 2$. Then, the topological invariant constraint is defined as follows:
\begin{align}
    \label{eq:topological invarinat constraint}
    -M  \leq -\frac{I_{\mathrm{nce}} + I'_{\mathrm{nce}}}{2} \leq -M + \delta,
\end{align}
where $\delta$ is a small fixed positive value, and $M = \sup_{\theta} \left(I_{\mathrm{nce}} + I'_{\mathrm{nce}}\right) / 2$. 

In practice, given a mini-batch $\mathcal{B} \subseteq \mathcal{D}$, 
we can approximate $-\left(I_{\mathrm{nce}} + I'_{\mathrm{nce}}\right) / 2$ by $-(\hat{I}_{\text{nce}} + \hat{I}'_{\text{nce}}) / 2$ (recall Eq.\eqref{eq:estimated infonce} for $\hat{I}_{\rm nce}$), where  $-\hat{I}_{\text{nce}}$ is given by 
\begin{align}
\label{eq:empirical i_nce with T_s}
    -\hat{I}_{\text{nce}} = -\frac{1}{|\mathcal{B}|}\sum_{x_i \in \mathcal{B}}\log\frac{e^{q\left(g_{\theta}(x_i), g_{\theta}\left(t_i(x_i)\right)\right)}}{\frac{1}{|\mathcal{B}|}\sum_{x_j \in \mathcal{B}}e^{q\left(g_{\theta}(x_i), g_{\theta}\left(t_j(x_j)\right)\right)}}
\end{align}
with the sampled transformation function 
$t_i$ from $p(t|x_i)$
and $-\hat{I}'_{\text{nce}}$ is 
given by switching two inputs in the function $q$ of Eq.\eqref{eq:empirical i_nce with T_s}. Here, 
$|\mathcal{B}|$ denotes the cardinality of $\mathcal{B}$.

\begin{figure}[!t]
\centering
\includegraphics{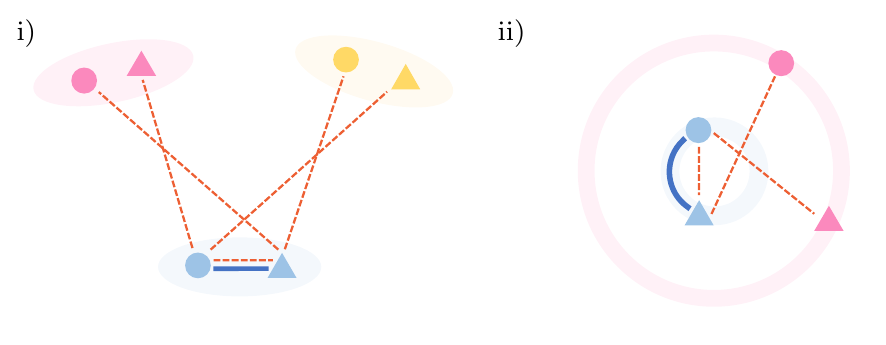}
\caption{
Illustration of the effect by minimizing point-wise positive loss $\ell_{\rm ps}(x_i)$ and point-wise negative loss $\ell_{\rm ng}(x_i)$.
In both i) and ii), the colors (blue, magenta, and yellow) mean different labels, and the light-colored manifolds express true clusters. For i) (resp. ii)), the set of clusters builds Three-Blobs (resp. Two-Rings), and it is an example of non-complex (resp. complex) topology datasets. A pair of the small circle and triangle symbols with the same color means a pair of a data point $x$ and the transformed data point $t(x)$, where such pair is constructed by $\mathcal{T}_{\mathfrak{e}}$ of Definition~\ref{def:process T_e} (resp. $\mathcal{T}_{\mathfrak{g}}$ of Definition~\ref{def:process T_g}) in i) (resp. ii)).
The two data points connected by the red dash line (resp. blue straight or curved line) are enforced to be distant (resp. close) to each other by minimizing $\ell_{\rm ng}(x_i)$ (resp. $\ell_{\rm ps}(x_i)$).
}\label{fig:symmetric InfoNCE}
\end{figure}

To understand the above empirical symmetric InfoNCE more, 
we decompose 
$-(\hat{I}_{\text{nce}} + \hat{I}'_{\text{nce}}) / 2$
into the following three terms:
\begin{equation}
\label{eq:metric learning decomposition}
    \begin{split}
        -\log|\mathcal{B}| \underbrace{-\frac{1}{|\mathcal{B}|}\sum_{x_i \in \mathcal{B}}q\left(g_{\theta}(x_i), g_{\theta}\left(t_i(x_i)\right)\right)}_\text{$L_{\rm ps}$: positive loss} 
        \underbrace{+\frac{1}{|\mathcal{B}|}\sum_{x_i \in \mathcal{B}}\frac{1}{2}\log\left(\sum_{x_j \in \mathcal{B}}\sum_{x_{i'} \in \mathcal{B}}e^{a\left(i,i',j\right)}\right)}_\text{$L_{\rm ng}$: negative loss}, 
    \end{split}
\end{equation}
where $a\left(i,i',j\right)=q\left(g_{\theta}(x_i), g_{\theta}\left(t_j(x_j)\right)\right) + q\left(g_{\theta}(x_{i'}), g_{\theta}\left(t_i(x_i)\right)\right)$. Note that the decomposition is based on the fact that $q(z,z') = q(z',z)$ for all $z, z' \in \Delta^{C-1}$.
In Eq.\eqref{eq:metric learning decomposition}, we call the second and the third term \emph{positive loss} and 
\emph{negative loss}, respectively. 
These names are natural in the sense of metric learning~\citep{sohn2016improved}.
Indeed, 
making $L_{\rm ps}$ smaller w.r.t. $\theta$ leads to 
$g_{\theta}(x_i)\approx g_{\theta}\left(t_i(x_i)\right)$ for all $i$ due to the definition of $q$. %
Thus, since $t_i(x_i)$ is a neighbor data point of $x_i$, 
via minimization of $L_{\rm ps}$, the model predicts
the same cluster labels for $x_i$ and $t_i(x_i)$.
Here, the neighborhood is defined with the Euclidean (resp. the geodesical) neighborhood on K-NN graph of $\mathcal{D}$ through $\mathcal{T}_{\mathfrak{e}}$ (resp. $\mathcal{T}_{\mathfrak{g}}$). 
On the other hand, making $L_{\rm ng}$ smaller leads to $g_{\theta}(x_i) \neq g_{\theta}\left(t_j(x_j)\right)$ for all $i,j$ and $g_{\theta}(x_{i'}) \neq g_{\theta}\left(t_i(x_i)\right)$ for all $i'$ and $i$ due to the definition of $q$. %
Thus, 
via minimization of $L_{\rm ng}$, the model can return non-degenerate clusters (i.e., not all the predicted cluster labels are the same).

In Figure~\ref{fig:symmetric InfoNCE}, for simplicity, we illustrate effects brought by minimizing point-wise positive loss $\ell_{\rm ps}(x_i)$ and point-wise negative loss $\ell_{\rm ng}(x_i)$, which are defined in Eq.\eqref{eq:metric learning decomposition} as follows;
$
    \ell_{\rm ps}(x_i) = -q\left(g_{\theta}(x_i), g_{\theta}\left(t_i(x_i)\right)\right),\;\ell_{\rm ng}(x_i)=\frac{1}{2}\log\left(\sum_{x_j \in \mathcal{B}}\sum_{x_{i'} \in \mathcal{B}}e^{a\left(i,i',j\right)}\right).
$

\subsection{Theoretical Analysis}
\label{subsubsec:theoretical properties}
In this section, we investigate theoretical properties of the symmetric InfoNCE loss. 
In Section~\ref{subsubsubsec: analysis from viewpoints of mi}, we study the relationship between MI and symmetric InfoNCE. %
In Section~\ref{subsubsubsec:futher motivations}, we show a theoretical difference between InfoNCE and symmetric InfoNCE.

\subsubsection{Relationship between Symmetric InfoNCE and MI}
\label{subsubsubsec: analysis from viewpoints of mi}
First we make clear the reason for selecting Eq.\eqref{eqn:alpha-tau-model} as a critic function.
Our explanation begins by deriving the optimal critic of the symmetric InfoNCE loss.
\begin{prop}
\label{prop: relation between mi and symmetric infonce and optimal critic}
Let $Z$ and $Z'$ denote two random variables having the joint probability density $p$. Let $I_{\mathrm{nce},q}(Z;Z')$ the InfoNCE loss defined in Eq.\eqref{eq:infonce}. Let us define $I_{\mathrm{nce},q}(Z';Z)$ by switching $Z$ and $Z'$
of $I_{\mathrm{nce},q}(Z;Z')$. Then, the following MI, 
$I(Z;Z') := \mathbb{E}_{p(Z,Z')}\left[\log \frac{p(Z,Z')}{p(Z)p(Z')}\right],$
is an upper bound of the symmetric InfoNCE,  $\frac{I_{\mathrm{nce},q}(Z;Z') + I_{\mathrm{nce},q}(Z';Z)}{2}$. 
Moreover, if the function $q$ satisfies
\begin{equation}
\label{eq:optimal critic with symmetric infonce}
    q(z,z')=\log{\frac{p(z,z')}{p(z)p(z')}}+c,\,c\in\mathbb{R},
\end{equation}
then the equality
$I(Z;Z') = 
\frac{I_{\mathrm{nce},q}(Z;Z') + I_{\mathrm{nce},q}(Z';Z)}{2}$ holds. In other words, $q$ satisfying Eq.\eqref{eq:optimal critic with symmetric infonce} is the optimal critic.
\end{prop}
The proof is shown in Appendix~\ref{append: proof of relation between mi and symmetric infonce and optimal critic}.

Consider $g_\theta(X)$ and $g_\theta(T(X))$ as $Z$ and $Z'$ of Proposition~\ref{prop: relation between mi and symmetric infonce and optimal critic}, respectively. Then, the symmetric InfoNCE  $\left(I_{\mathrm{nce}} + I'_{\mathrm{nce}}\right) / 2$ of Eq.\eqref{eq:topological invarinat constraint} can be upper-bounded by 
\begin{equation}
\label{eq:upper-bound of proposed symmetric infonce}
    I(g_{\theta}(X);g_{\theta}(T(X))).
\end{equation}
Thus, maximization of the symmetric InfoNCE (i.e., the constraint of Eq.\eqref{eq:topological invarinat constraint}) is a reasonable approach to maximize the MI. 
Note that the computation of $I(g_{\theta}(X);g_{\theta}(T(X)))$ is difficult, since density-estimation on $\Delta^{C-1}$ is required.

It is interesting that the optimal critic of the symmetric InfoNCE loss is the pointwise MI of $g_{\theta}(X)$ and $g_{\theta}(T(X))$ up to an additive constant.
Moreover, we remark that the function $q$ of Eq.\eqref{eqn:alpha-tau-model} is in fact designed based on the optimal critic, Eq.\eqref{eq:optimal critic with symmetric infonce}, of the symmetric InfoNCE.
As shown in the equation, 
the optimal critic of the symmetric InfoNCE is $q^\ast(z,z') = \log\frac{p(z,z')}{p(z)p(z')} + c, c\in\mathbb{R}$. 
Thus, the joint probability density $p(z,z')$ is expressed by 
$p(z,z') \propto p(z)p(z')e^{q^\ast(z,z')}$. 
Hence, the critic function adjusts the statistical dependency between $z$ and $z'$. In our study, we suppose that $z, z' \in \Delta^{C-1}$, and
the critic $q(z,z')$ is expressed 
as an increasing function of $z^\top z'$. 
When $z$ and $z'$ are both the same one-hot vector in $\Delta^{C-1}$, $p(z,z')$ is assumed to be large. 
On the other hand, if $z^\top z' = 0$, $p(z,z')$ is assumed to take a small value. 
We also introduce a one-dimensional parameter $\alpha$ for the critic $q_{\alpha}$ to tune the intensity of the dependency. 
Although there are many choices of critic functions, we here employ the $\alpha$-exponential function, 
because $\exp_{\alpha}$ can express a wide range of common probabilities in statistics only by one parameter; see details of $\alpha$-exponential function in~\citet{naudts2009q,amari2011geometry,matsuzoe2012geometry}. 
Eventually, the model of the critic is given by $p_{\alpha}(z,z')\propto p(z) p(z')
    \exp_{\alpha}\big(\tau(z^\top z'-1)\big)$,
where $\alpha \in \mathbb{R}$ and $0 \leq \tau \leq |1-\alpha|^{-1}$.
Note that the normalization constant of $p_{\alpha}$ is no need when we compute the symmetric InfoNCE. In our experiments, we consider both $\alpha$ and $\tau$ as the hyper-parameters.

\begin{rmk}
The cosine-similarity function 
$s(z, z') = z^\top z' / \|z\|_2 \|z'\|_2, z,z'\in\Delta^{C-1}$ is commonly used in the context of representation learning~\citep{chen2020simple,bai2021connecting}. 
However, we do not use the cosine-similarity function 
as the critic function $q$ in  Eq.\eqref{eqn:alpha-tau-model}. 
This is because in our problem the cosine-similarity function is not relevant to estimate the one-hot vector by the model $g_{\theta}(x)$. 
Indeed, for $q(z,z') = s(z,z')$ and $C=2$, the pair $(z,z')$ satisfying $z=z'=(1/2,1/2)^\top \in \Delta^{C-1}$ is a maximizer of $q(z, z^\prime)$, i.e., there exists a pair of non-one-hot vectors $z$ and $z'$ that minimizes $L_{\rm ps}$
in Eq.\eqref{eq:metric learning decomposition}. 
\end{rmk}

Next, we investigate a few more properties of the symmetric InfoNCE loss from the perspective of MI.
First we present a theoretical comparison between the symmetric InfoNCE loss and IIC (see IIC in Section~\ref{subsec: representatitve deep methods in scenario2} and Appendix~\ref{subsec:Invariant Information Clustering}).

\begin{prop}
\label{prop:proposed vs iic}
Consider a feature $X$ and its transformation function $T$ of Definition~\ref{def:random transformation function}. Let $Y$ (resp. $Y'$) denote a cluster label of $X$ (resp. $T(X)$). 
Then, 
the following inequality holds:
\begin{equation}
\label{eq:iic vs proposed}
    I(Y;Y')\leq I(g_{\theta}(X);g_{\theta}(T(X))),
\end{equation}
where $g_{\theta}$ is the same model as introduce in Definition~\ref{def:g_theta}. %
\end{prop}
The proof is shown in  Appendix~\ref{append: proof of iic vs proposed}.

The above data processing inequality guarantees that $I(g_{\theta}(X);g_{\theta}(T(X)))$ brings richer information than $I(Y;Y')$ used in IIC. Since our constraint is related to  $I(g_{\theta}(X);g_{\theta}(T(X)))$, Eq.\eqref{eq:iic vs proposed} indicates the advantage of ours over IIC.
To discuss the consequence of Proposition~\ref{prop:proposed vs iic} in more detail, 
we provide a statistical analysis on the gap between the following two quantities:
\begin{itemize}
    \item[1)] The maximum value $I(g_{\theta}(X);g_{\theta}(T(X)))$ w.r.t. $\theta$,
    \item[2)] The mutual information evaluated at $\widehat{\theta}$, where $\widehat{\theta}$ is the parameter maximizing the empirical symmetric InfoNCE.
\end{itemize}
To the best of our knowledge, such statistical analysis is not provided in previous theoretical studies related to InfoNCE.

\begin{thm}[Informal version]
    \label{thm:Estimation Error Analysis Informal}
    Consider the empirical symmetric InfoNCE of Section~\ref{subsec:symmetric infonce} 
    with a critic $q \in \mathcal{Q}$
    for a dataset $\mathcal{D}=\{x_i\}_{i=1}^n$.  %
    Here, $\mathcal{Q}$ is a set of critics
    defined as follows: $\mathcal{Q}=\{\phi_{(\alpha,\tau)}\,:\,(\alpha,\tau)\in\Xi\}$, where $\phi_{(\alpha,\tau)}(z^{\top}z')=\log\left(\exp_{\alpha}\big(\tau(z^\top z'-1)\big)\right)$ (see Eq.\eqref{eqn:alpha-tau-model}), and $\Xi$ is a set of all possible $(\alpha,\tau)$ pairs.
    Let $\widehat{I}_{\mathrm{sym\_nce},q}(\theta)$ denote the empirical symmetric InfoNCE, where $\theta$ is a set of parameters in $g_\theta$ of Definition~\ref{def:g_theta}.
    Let us define $\widehat{\theta}$ by %
    $
        \widehat{\theta} = \arg\max_\theta\sup_{q\in\mathcal{Q}}\widehat{I}_{\mathrm{sym\_nce},q}(\theta). %
    $
   We define $\theta^*$ as the 
   maximizer of 
   $I(g_{\theta}(X);g_{\theta}(T(X)))$ w.r.t $\theta$.
   Suppose that $0\leq\delta$ is a constant. %
    Then, with the probability greater than $1-\delta$, the gap between $I(g_{\theta^*}(X);g_{\theta^*}(T(X)))$ and $I(g_{\widehat{\theta}}(X);g_{\widehat{\theta}}(T(X)))$ is given by 
  \begin{align}
  \label{eq: estimation error bound}
  \begin{split}
    &\phantom{\leq} I(g_{\theta^*}(X);g_{\theta^*}(T(X)))-I(g_{\widehat{\theta}}(X);g_{\widehat{\theta}}(T(X)))\\
    &\leq \mathrm{(Approx.~Err.)} + (\mathrm{Gen.~Err.})+c\,\sqrt{\frac{\log(1/\delta)}{n}},
  \end{split}
  \end{align}
   where $c > 0$ is a constant, and Approx. Err. (resp. Gen. Err.) is short for Approximation Error (resp. Generalization Error).
   Note that the generalization error term $\mathrm{(Gen.~Err.)}$  
   consists of Rademacher complexities with a set of neural network models.
\end{thm}
See Appendix~\ref{append:Estimation Error of InfoNCE} for the proof of the formal version.

From Theorem~\ref{thm:Estimation Error Analysis Informal}, the gap indeed gets close if the following A1) and A2) hold:
\begin{itemize}
    \item[A1)] $(\mathrm{Approx.~Err.})$ of Eq.\eqref{eq: estimation error bound} is small (i.e., the set $\mathcal{Q}$ contains a rich quantity of critic functions). %
    \item[A2)] $(\mathrm{Gen.~Err.})$ and $\sqrt{\frac{\log(1/\delta)}{n}}$ of Eq.\eqref{eq: estimation error bound} are small.
\end{itemize}
It is known that the Rademacher complexity of a kind of neural network models is $O(n^{-1/2})$; see~\citet{bartlett02:_radem_gauss_compl}. Thus, the condition A2) can be satisfied if the sample size $n$ is large enough.
Moreover, by combining Proposition~\ref{prop:proposed vs iic} with Theorem~\ref{thm:Estimation Error Analysis Informal}, we obtain the following implication:
if $n$ is sufficiently large, then the gap between the MI, $I(g_{\theta^{*}}(X);g_{\theta^{*}}(T(X)))$, and the plug-in estimator with the optimal estimator $\widehat{\theta}$ of the empirical symmetric InfoNCE is reduced.
On the other hand, from Proposition~\ref{prop:proposed vs iic}, the MI of the pair $Y$ and $Y'$ is always less than or equal to that of the pair $g_{\theta}(X)$ and $g_{\theta}(T(X))$.
Since IIC is an empirical estimator of the MI,  $I(Y,Y')$, the statistical dependency via MI of the probability vectors $g_{\theta}(X)$ and $g_{\theta}(T(X))$ obtained by optimizing the symmetric InfoNCE can be greater than that of $Y$ and $Y'$ learned through the optimization of IIC.
Therefore, the symmetric InfoNCE has a more potential to work as a topologically invariant constraint for deep clustering than other MIs such as IIC.

Note that in the almost same way as Theorem~\ref{thm:Estimation Error Analysis Informal}, it is possible to derive a similar result for the gap between the following two: 1) $I(g_{\theta^*}(X);g_{\theta^*}(T(X)))$ and 2) $\max_\theta\sup_{q\in\mathcal{Q}}\widehat{I}_{\mathrm{sym\_nce},q}(\theta)$.
This fact indicates that if the upper bound derived in a similar way to Eq.\eqref{eq: estimation error bound} is small enough, then the empirical symmetric InfoNCE has a potential to strengthen the  dependency between $g_{\theta}(X)$ and $g_{\theta}(T(X))$.

\subsubsection{Further Motivations behind the Symmetric InfoNCE Loss}
\label{subsubsubsec:futher motivations}

We also leverage the theoretical result on contrastive representation learning from \citet{wang2022chaos}, in order to explain the difference between InfoNCE and symmetric InfoNCE.
\begin{thm}
\label{thm: downstream tasks guarantee analysis}
Let us define $X\in\mathbb{R}^d$ and $Y\in\{1,\cdots,C\}$ as described in Section~\ref{subsec:Notations}. %
Let $Z= g_\theta(X)$ and $Z' = g_\theta(T(X))$, where $g_\theta:\mathbb{R}^d\to\Delta^{C-1}$ and $T:\mathbb{R}^d\to\mathbb{R}^d$ are given by Definition~\ref{def:g_theta} and~\ref{def:random transformation function}, respectively. 
The symmetric InfoNCE $(I_{\mathrm{nce}}+I_{\mathrm{nce}}')/2$ of Eq.\eqref{eq:topological invarinat constraint} is supposed to set $\alpha=1$ and a fixed $\tau$ for the critic function of Eq.\eqref{eqn:alpha-tau-model}. %
Assume that $p(Y)$ is a uniform distribution.
Let $\mathcal{L}_{\mathrm{CE,Raw}}^{\widetilde{\mu}}(g_\theta)$ denote the mean supervised loss, which is given by $\mathcal{L}_{\mathrm{CE,Raw}}^{\widetilde{\mu}}(g_\theta)=-\mathbb{E}_{p(Z,Y)}\left[\log{\frac{\exp(Z^{\top}\widetilde{\mu}_{Y})}{\sum_{k=1}^{C}\exp(Z^{\top}\widetilde{\mu}_{k})}}\right],$
where $\widetilde{\mu}_{k}=\tau\cdot\mathbb{E}_{p(Z'|Y=k)}[Z'], k\in\{1,\cdots,C\}$.
In other words, $\mathcal{L}_{\mathrm{CE,Raw}}^{\widetilde{\mu}}(g_\theta)$ is the cross-entropy loss via a linear evaluation layer, whose parameters are  $\widetilde{\mu}=(\widetilde{\mu}_{1},\cdots,\widetilde{\mu}_{C})\in\mathbb{R}^{C\times C}$.
Similarly we define  $\mathcal{L}_{\mathrm{CE,Aug}}^{\widetilde{\mu}}(g_\theta)$ by $\mathcal{L}_{\mathrm{CE,Aug}}^{\mu}(g_\theta)=-\mathbb{E}_{p(Z',Y)}\left[\log{\frac{\exp({Z'}^{\top}\mu_{Y})}{\sum_{k=1}^{C}\exp({Z'}^{\top}\mu_{k})}}\right],$
where $\mu_{k}=\tau\cdot\mathbb{E}_{p(Z|Y=k)}[Z], k\in\{1,\cdots,C\}$, and $\mu=(\mu_{1},\cdots,\mu_{C})$.
Let us introduce the symmetric mean supervised loss as $\mathcal{L}_{\mathrm{SCE}}^{\mu,\widetilde{\mu}}(g_\theta)=(\mathcal{L}_{\mathrm{CE,Raw}}^{\widetilde{\mu}}(g_\theta)+\mathcal{L}_{\mathrm{CE,Aug}}^{\mu}(g_\theta))/2$.
Then, we have
\begin{align*}
    &-\frac{I_{\mathrm{nce}}+I_{\mathrm{nce}}'}{2} - \frac{1}{2}\left( \sqrt{\mathrm{Var}(Z|Y)}+\sqrt{\mathrm{Var}(Z'|Y)}\right) - \frac{1}{2e}\mathrm{Var}(\exp({\tau}Z^{\top}Z'))\\
    &\leq 
    \mathcal{L}_{\mathrm{SCE}}^{\mu,\widetilde{\mu}}(g_\theta)-\log{C}\\
    &\leq -\frac{I_{\mathrm{nce}}+I_{\mathrm{nce}}'}{2} +\frac{1}{2}\left(\sqrt{\mathrm{Var}(Z|Y)}+\sqrt{\mathrm{Var}(Z'|Y)}\right),
\end{align*}
where 
\begin{align*}
    \mathrm{Var}(Z|Y) &= \mathbb{E}_{p(Y)}[\mathbb{E}_{p(Z|Y)}[\|{\tau}Z-\mu_{Y}\|_{\infty}^{2}]], \\
    \mathrm{Var}(Z'|Y) &= \mathbb{E}_{p(Y)}[\mathbb{E}_{p(Z'|Y)}[\|{\tau}Z'-\mu_{Y}\|_{\infty}^{2}]], \\
    \mathrm{Var}(\exp(\tau Z^{\top}Z')) &= \mathbb{E}_{p(Z)p(Z')}[(\exp(\tau Z^{\top}Z')-\mathbb{E}_{p(Z)p(Z')}[\exp(\tau Z^{\top}Z')])^{2}].
\end{align*}
\end{thm}
The proof is shown in Appendix~\ref{append: Downstream Tasks Guarantee}.

\begin{rmk}
In Theorem~\ref{thm: downstream tasks guarantee analysis}, 
the critic function with $\alpha=1$ is considered for the sake of simplicity. 
We can derive almost the same upper and lower bounds 
for the symmetric InfoNCE using the critic of Eq.\eqref{eqn:alpha-tau-model} with 
$\alpha$ such that $1-1/\tau<\alpha<1$.
The proof is the same as that of Theorem~\ref{thm: downstream tasks guarantee analysis}. 
We use the concavity of the function $u\mapsto\log\exp_{\alpha}(u)$ and 
the inequality $\log\exp_{\alpha}(x+y)\leq\log\exp_{\alpha}(x)+|y|/(1-(1-\alpha)\tau)$
for $x,x+y\in[-\tau,0]$. 
\end{rmk}

Our result includes four technical differences and modifications from \citet{wang2022chaos} as follows:
    1) Theorem~\ref{thm: downstream tasks guarantee analysis} is intended for the symmetric InfoNCE loss.
    2) We do not assume that any positive pair $(Z,Z')\sim p(Z,Z')$ has the identical label distribution given the representation (i.e., we do not rely on the assumption $p(Y|Z)=p(Y|Z')$).
    Note that the assumption of $p(Y|Z)=p(Y|Z')$ will not hold in practical settings. 
    For instance, suppose that we have an image $X$.
    If $X$ is cropped, then the cropped image $T(X)$ may have lost some information included in $X$, which would result in the case where the distribution of $X$ and that of $T(X)$ do not agree.
    3) In the proof of Theorem~\ref{thm: downstream tasks guarantee analysis} (see Proposition~\ref{prop: lower bound for downstream tasks guarantee} in Appendix~\ref{append: Downstream Tasks Guarantee}), we use the sharpened Jensen's inequality~\citep{liao2018sharpening} in order to make our proof simpler. On the other hand, 
    Theorem 4.2 of~\citet{wang2022chaos} is obtained by utilizing Corollary~3.5 of \citet{budimir2000further}. 
    4) We consider the case in which the distribution of a random variable representing unlabeled data and one of its augmentation data are the same. 
    In our setup, if $p(Z, Y)=p(Z', Y)$ holds, then we have  $\mathcal{L}_{\mathrm{CE,Raw}}^{\widetilde{\mu}}(g_\theta)=\mathcal{L}_{\mathrm{CE,Aug}}^{\mu}(g_\theta)$.
    In general, however, the probability distribution of $Z$ and $Z'$ are not necessarily the same.
    More precisely, let $(\Omega,\mathcal{F},P)$ be a probability space and $X$ be a random variable on $\Omega$.
    Then let us consider the push-forward distribution $Z_{\#}P$ and $Z'_{\#}P$.
    Since the transformation map $T$ is also a random variable, generally these distributions are distinct from each other.
    We avoid this issue by starting from the general setting. 
    
    Furthermore, our result gives the following novel insight into the theoretical understanding of the symmetric InfoNCE: 
    the symmetric InfoNCE reduces both $\mathcal{L}_{\mathrm{CE,Raw}}^{\widetilde{\mu}}(g_\theta)$ and $\mathcal{L}_{\mathrm{CE,Aug}}^{\mu}(g_\theta)$ at the same time.
    This property could explain why the symmetric InfoNCE performs more stable in practice than InfoNCE as a constraint of deep clustering methods: 
    see also Table~\ref{tb:results of clustering accuracy} that shows the comparison of InfoNCE (MIST via $\hat{I}_{\mathrm{nce}}$) and symmetric InfoNCE (MIST).

    For further comparison between symmetric InfoNCE, InfoNCE, and SimCLR~\citep{chen2020simple}, see Appendix~\ref{append: Comparison between Symmetric InfoNCE, InfoNCE, and SimCLR}.

\section{Numerical Experiments}
\label{sec:numerical-experiments}
Throughout this section, we aim to evaluate the efficiency of 
the symmetric InfoNCE as topological invariant constraint for a deep clustering method.
To this end, at first in Section~\ref{subsec:application}, we define a deep clustering method of \textsf{Scenario2} named MIST by applying the symmetric InfoNCE to IMSAT~\citep{hu2017}. 
The reason why we employ IMSAT is that it performs the best on average among %
deep clustering methods in Table~\ref{tab:types of deep clustering methods}. 
Then, in Section~\ref{subsec:Results}, we compare MIST and IMSAT in terms of clustering accuracy to observe the benefits of the symmetric InfoNCE, while comparing MIST with the other representative  methods as well.
Thereafter in Section~\ref{subsec: ablation studies for mist objective}, we conduct ablation studies on MIST objective to understand the effect of each term in Eq.\eqref{eq:rewritten implementing unconstrained objective}. At last in Section~\ref{subsec: robustness k0 and alpha}, using MIST, we examine robustness of important hyper-parameters in the symmetric InfoNCE.

\subsection{MIST: Application of Symmetric InfoNCE to IMSAT}
\label{subsec:application}
Given a mini-batch $\mathcal{B} \subseteq \mathcal{D}$, by applying our empirical symmetric InfoNCE  of Eq.\eqref{eq:metric learning decomposition} to the objective of IMSAT (see the objective in Eq.\eqref{eq:practical imsat objective} of Appendix~\ref{subsubsec:Objective of IMSAT}), we define the following objective of MIST:  
\begin{equation}
\begin{split}
\label{eq:rewritten implementing unconstrained objective}
    \theta^\ast = {\rm arg}{\rm min}_{\theta} \Biggl[ \underbrace{R_{\rm vat}\left(\mathcal{B};\theta\right)}_{\ctext{A}} 
                - \mu\biggl\{\eta \underbrace{H(Y)}_{\ctext{B}} - \underbrace{H(Y|X)}_{\ctext{C}} - \gamma \underbrace{\left(L_{\rm ps} + L_{\rm ng}\right)}_{\ctext{D}} \biggr\} \Biggr],
\end{split}
\end{equation}
where $\mu$, $\eta$ and $\gamma$ are positive hyper-parameters. 
The symbol \ctext{A} expresses VAT (Virtual Adversarial Training) loss~\citep{miyato2018virtual}; see~Eq.\eqref{eq: vat-loss} of Appendix~\ref{subsubsec:vat}. In addition, \ctext{B} and \ctext{C} mean Shannon entropy~\citep{cover1999elements} w.r.t. a cluster label $Y$ and conditional entropy of $Y$ given a feature $X$, respectively. Moreover, minimization of the symbol \ctext{D} is equivalent to maximization of the empirical symmetric InfoNCE.
Note that the major difference between MIST and IMSAT is the introduction of term \ctext{D}. 
The minimization problem of Eq.\eqref{eq:rewritten implementing unconstrained objective} is solved via SGD (Stochastic Gradient Descent)~\citep{shalev2014understanding} in our numerical experiments. 
See Appendix~\ref{append:details of mist objective} for further details of MIST objective, the pseudo algorithm (Algorithm~\ref{alg:proposed}), and the diagram (Table~\ref{fig:MIST diagram}).

\subsection{Dataset Description and Evaluation Metric}
\label{subsec:dataset description and evaluation metric}
We use two synthetic datasets and eight real-world benchmark datasets in our experiments. All the ten datasets are given as feature vectors. For the synthetic datasets, we employ Two-Moons and Two-Rings of  scikit-learn~\citep{geron2019hands}. The real-world datasets are MNIST~\citep{lecun1998gradient}, SVHN~\citep{netzer2011reading}, STL~\citep{coates2011analysis}, CIFAR10~\citep{torralba200880}, CIFAR100~\citep{torralba200880}, Omniglot~\citep{lake2011one}, 20news~\citep{lang1995newsweeder} and Reuters10K~\citep{lewis2004rcv1}. The first six real-world datasets originally belong to 
the image domain and the last two originally belong to the text domain. As for the characteristic of each dataset, Two-Moons and Two-Rings are low-dimensional datasets with complex topology.  MNIST, STL, and CIFAR10 are  balanced datasets with the small number of clusters. CIFAR100, Omniglot, and 20news are balanced datasets with the large number of clusters. SVHN and Reuters10K are imbalanced datasets. For further details of the above  ten datasets, see Appendix~\ref{append:Detail of Dataset}.

In the unsupervised learning scenario, 
we adopt the standard metric for evaluating clustering performance, which measures how close the estimated cluster labels are to the ground truth under a permutation. For an unlabeled dataset $\left\{x_{i}\right\}_{i=1}^{n}$, let $\left\{y_{i}\right\}_{i=1}^{n}$ and $\left\{\hat{y}_{i}\right\}_{i=1}^{n}$ be its true cluster label set and estimated cluster label set, respectively.  %
Suppose that the both true $y_{i}$ and estimated cluster labels $\hat{y}_{i}$ take the same range $\{1,\cdots,C\}$. The \textit{clustering accuracy} ACC is defined by
$\mathrm{ACC}\;(\%) = 100 \times \max _{\sigma} \frac{\sum_{i=1}^{n} \mathbb{I}\left[y_{i}=\sigma\left(\hat{y}_{i}\right)\right]}{n}$,
where $\sigma$ ranges over all permutations of cluster labels, and $\mathbb{I}[\;\cdot\;]$ is the indicator function. 
The optimal assignment of $\sigma$ can be computed using the Kuhn-Munkres algorithm~\citep{kuhn1955hungarian}.

\subsection{Statistical Model and Optimization}
\label{subsec:Statistical Model and Optimization}
Throughout all our experiments, we fix our clustering neural network model %
$g_{\theta}(x) \in \Delta^{C-1}$
to the following simple and commonly used MLP architecture with softmax~\citep{hinton2012improving}: $d-1200-1200-C$, where $d$ is the dimension of the feature vector. 
We apply the ReLU activation function~\citep{nair2010rectified} and BatchNorm~\citep{ioffe2015batch} for all hidden layers. In addition, the initial set with $\theta$ is defined by He-initialization~\citep{he2015delving}. For optimizing the model, we employ Adam optimizer~\citep{kingma2014adam}, and set $0.002$ as the learning rate.

We implemented MIST\footnote{\url{https://github.com/betairylia/MIST} [Last accessed 23-July-2022]}\label{footnote}
 using Python with PyTorch library~\citep{ketkar2017introduction}.
All experiments are evaluated with NVIDIA TITAN RTX GPU, which has a 24GiB GDDR6 video memory. %

\subsection{Compared Methods}
\label{subsec:Compared Methods}
As baseline methods, we employ the following three classical clustering methods: K-means \citep{macqueen1967some}, SC~\citep{ng2001spectral} and GMMC~\citep{day1969estimating}. 
For deep clustering methods, 
we employ representative deep clustering methods from 
$\mathfrak{T}_1$ to $\mathfrak{T}_6$ of Table~\ref{tab:types of deep clustering methods}, MIST via $\hat{I}_{\rm nce}$, and MIST (our method) of Eq.\eqref{eq:rewritten implementing unconstrained objective}. Here, MIST via $\hat{I}_{\rm nce}$ is defined by replacing $-(L_{\rm ps} + L_{\rm ng})$ of Eq.\eqref{eq:rewritten implementing unconstrained objective} by $-\hat{I}_{\text{nce}}$ of Eq.\eqref{eq:empirical i_nce with T_s}. The reason why MIST via $\hat{I}_{\rm nce}$ is employed is to check how much more efficiently symmetric InfoNCE can enhance a deep clustering method over the original InfoNCE.
In both MIST and MIST via $\hat{I}_{\rm nce}$, $\mathcal{T}_{\mathfrak{g}}$ of Definition~\ref{def:process T_g} (resp. $\mathcal{T}_{\mathfrak{e}}$ of Definition~\ref{def:process T_e}) is employed for synthetic datasets (resp. real-world datasets).
For further details of hyper-parameter tuning with MIST and MIST via $\hat{I}_{\rm nce}$, see Appendix~\ref{subsec:Hyperparameter Setting with Proposed Method}.
Moreover, from $\mathfrak{T}_2$, $\mathfrak{T}_3$, 
$\mathfrak{T}_5$, and $\mathfrak{T}_6$,
SpectralNet~\citep{shaham2018},
VaDE~\citep{jiang2016variational},
CatGAN~\citep{springenberg2015unsupervised} and 
SELA~\citep{asano2019self} are respectively examined. 
From $\mathfrak{T}_1$, DEC~\citep{xie2016unsupervised} and SCAN~\citep{van2020scan} are examined.
From $\mathfrak{T}_4$, IMSAT~\citep{hu2017} and IIC~\citep{ji2019invariant} are examined.
Note that  SCAN, IIC, CatGAN, and SELA were originally proposed in \textsf{Scenario1} of 
Section~\ref{subsec:our scenario}. %
Therefore, we redefine those methods to make them fit to  \textsf{Scenario2} in our experiments. 
The redefinitions and implementation details of all   the existing methods are  described in Appendix~\ref{append:Implementation Details with Previous Methods}

\subsection{Analysis from Table~\ref{tb:results of clustering accuracy}}
\label{subsec:Results}

As briefly explained in Section~\ref{subsec:motivation},
the average clustering accuracy and its standard deviation 
on each dataset for the corresponding  clustering method are reported in Table~\ref{tb:results of clustering accuracy}.
At first, since MIST clearly outperforms IMSAT for almost all the dataset, we can observe benefit of the symmetric InfoNCE. 
Especially for Two-Moons and Two-Rings (two complex topology datasets), it should be emphasized that the symmetric InfoNCE with $\mathcal{T}_{\mathfrak{g}}$ of Definition~\ref{def:process T_g} brings significant enhancement to IMSAT. In addition, for CIFAR10 and SVHN, it brings a noticeable gain to IMSAT. 

With comparison between MIST and SpectralNet, MIST cannot perform as stable as SpectralNet for Two-Rings dataset. 
However, MIST with a DNN needs a smaller memory complexity than SpectralNet with two DNNs. 
Moreover, the average performance of MIST on the eight real-world datasets are much better than that of SpectralNet.

Furthermore,  
through comparison between MIST and MIST via $\hat{I}_{\rm nce}$, 
we can observe that 
the symmetric InfoNCE enhances IMSAT more than InfoNCE does on average.
The observation matches Theorem~\ref{thm: downstream tasks guarantee analysis}.

\subsection{Ablation Study for MIST Objective}
\label{subsec: ablation studies for mist objective}
Recall \ctext{A} to \ctext{D} in Eq.\eqref{eq:rewritten implementing unconstrained objective}. Here, we examine six variants of MIST objective of Eq.\eqref{eq:rewritten implementing unconstrained objective}, which are shown in the first column of Table~\ref{table: ablation study result}.
For example, (\ctext{B}, \ctext{C}) means that only \ctext{B} and \ctext{C} 
are used to define a variant of the MIST objective, where \ctext{B} and \ctext{C} are linearly combined using a coefficient hyper-parameter. The detail of hyper-parameter tuning for each combination is described in Appendix~\ref{subsec:Hyperparameter Setting with Proposed Method}.
For the study, Two-Rings, MNIST, CIFAR10, 20news, and SVHN are employed.

\begin{table}[!t]
    \centering
    \caption{
    Results of the ablation study for MIST objective. The number outside (resp. inside) of brackets expresses clustering accuracy (resp. standard deviation) over three trials. In the first column, 
    six combinations based on \ctext{A} to \ctext{D} in Eq.\eqref{eq:rewritten implementing unconstrained objective} are shown, and each of the six defines a variant of MIST objective of Eq.\eqref{eq:rewritten implementing unconstrained objective}.
    }
    \vskip 0.5em
    \scalebox{1}{
    \begin{tabular}{cccccc}
    \hline
         & Two-Rings & MNIST & CIFAR10 & 20news & SVHN \\
         \hline
         (\ctext{D}) & 76.4(16.7) & 72.7(4.8) & 40.7(2.9) & 21.9(3.2) & 23.3(0.2) \\ %
         (\ctext{B}, \ctext{C})& 58.7(9.6) & 58.5(3.5) & 40.3(3.5)&25.1(2.8)&26.8(3.2)\\
         (\ctext{B}, \ctext{D}) & 83.4(23.5) & 81.9(4.3) & 44.1(0.5) & 40.1(1.1) & 24.9(0.2) \\ %
         (\ctext{A}, \ctext{D}) & 100(0) & 70.6(2.9) & 35.8(4.9) & 35.7(1.7) & 44.8(4.8) \\ %
         (\ctext{A}, \ctext{B}, \ctext{C}) & 69.0(21.9)&98.7(0.0) &44.9(0.6) & 35.8(1.9) &54.8(2.8) \\
         (\ctext{B}, \ctext{C}, \ctext{D}) & 83.4(23.4) & 75.0(4.3) & 45.1(1.8) & 31.6(0.4) & 21.0(2.5) \\ %
         \hline
    \end{tabular}
    }
    \label{table: ablation study result}
\end{table}

Firstly, by two comparisons of (\ctext{B}, \ctext{C}) vs. (\ctext{B}, \ctext{C}, \ctext{D}) and  (\ctext{A}, \ctext{B}, \ctext{C}) vs. MIST results in Table~\ref{tb:results of clustering accuracy}, we see positive effect of the symmetric InfoNCE
across the five datasets on average.
Especially for the complex topology dataset (i.e., Two-Rings), the effect is very positive. 
Secondly, %
the result of (\ctext{B}, \ctext{C}) vs. (\ctext{A}, \ctext{B}, \ctext{C}) indicates that 
VAT~\citep{miyato2018virtual} positively works for clustering tasks. 
Thirdly, via (\ctext{D}) vs. (\ctext{B}, \ctext{D}), effect of maximizing $H(Y)$ is positive. For further analysis with \ctext{A} to \ctext{D}, see Appendix~\ref{append:details on mist}.

To sum up, although the combination of (\ctext{A}, \ctext{B}, \ctext{C}), i.e., IMSAT, provides competitive clustering performance for non-complex topology datasets, the symmetric InfoNCE can bring benefit to the combination for not only the non-complex topology datasets but also the complex topology dataset.

\subsection{Robustness for $K_0$, $\alpha$, and $\gamma$}
\label{subsec: robustness k0 and alpha}

Let us consider the influence of the hyper-parameters, $K_0, \alpha$, and $\gamma$, in the MIST objective of 
Eq.\eqref{eq:rewritten implementing unconstrained objective} on the clustering performance. 
We evaluate how these hyper-parameters affect the clustering accuracy when other hyper-parameters are unchanged. 
In the study, some candidates of the three hyper-parameters are examined for Two-Rings, MNIST, CIFAR10, 20news, and SVHN.

\begin{table}[!t]
    \centering
    \caption{
    Results of robustness study for number of neighbors $K_0$ in Definition~\ref{def:process T_e} and~\ref{def:process T_g}. The average clustering accuracy and std over three trials are shown. In the first column, hyper-parameter value outside (resp. inside) of brackets is used for Two-Rings, MNIST, CIFAR10 and SVHN (resp. 20news).     }
    \vskip 0.5em
    \scalebox{1}{
    \begin{tabular}{cccccc}
    \hline
         $K_0$ & Two-Rings & MNIST & CIFAR10 & 20news & SVHN \\
         \hline
         $5(10)$&  83.5(23.4) & 98.2(0.4)& 48.0(0.9) & 34.2(1.5) & 55.1(2.0) \\
         $10(50)$& 83.5(23.3) & 96.6(5.7)& 47.5(0.9) & 36.5(1.0) & 55.9(1.7) \\
         $15(100)$& 100(0) & 98.4(0.0)& 47.8(1.4) & 38.8(0.9) & 56.3(3.2) \\
         $50(150)$& 50.7(0.5) & 93.6(7.2)& 48.6(1.8) & 36.9(2.2) & 63.3(1.2) \\
         \hline
    \end{tabular}
    }
    \label{table: robustness k0}
\end{table}

\begin{table}[!t]
    \centering
    \caption{
    Results of robustness study for $\alpha$ in Eq.\eqref{eqn:alpha-tau-model}. The average clustering accuracy and std over three trials are shown.
    }
    \vskip 0.5em
    \scalebox{1}{
    \begin{tabular}{cccccc}
    \hline
         $\alpha$  & Two-Rings & MNIST & CIFAR10 & 20news & SVHN \\
         \hline
         $0$& 67.2(23.2) & 98.7(0.0) & 49.5(0.3) & 38.1(1.8) & 61.4(2.1) \\
         $1$& 100(0) & 97.6(1.1) & 48.4(0.4) & 39.9(3.3) & 57.0(1.5) \\
         $2$& 66.7(23.5) & 97.8(1.3) & 46.6(0.4) & 39.5(2.5) & 57.6(2.4) \\
         \hline
    \end{tabular}
    }
    \label{table: robustness alpha}
\end{table}

\begin{table}[!t]
    \centering
    \caption{
    Results of robustness study for $\gamma$ in MIST objective of Eq.\eqref{eq:rewritten implementing unconstrained objective}.  The average clustering accuracy and std over three trials are shown. In the first column, number outside (resp. inside) of brackets means value of $\gamma$ used for real-world datasets (resp. synthetic dataset). 
    }
    \vskip 0.5em
    \scalebox{1}{
    \begin{tabular}{cccccc}
    \hline
         $\gamma$ & Two-Rings &  MNIST  & CIFAR10  & 20news  & SVHN   \\
         \hline
         $0.1 (1)$  & 83.8(22.9)  & 93.7(7.0) & 49.0(1.5)  & 40.7(1.1) & 52.5(4.6)  \\
         $0.5 (5)$  & 94.4(8.0)  & 98.0(1.0)& 48.7(1.0)  & 35.0(2.3) & 59.6(3.6)  \\
         $1.0 (10)$ & 100(0) &  97.9(1.1) & 46.5(0.8) & 37.4(0.7) & 59.8(3.5)  \\
         \hline
    \end{tabular}
    }
    \label{table: robustness gamma}
\end{table}

\begin{enumerate}
    \item[1)] The number of neighbors, $K_0$, is used in both $\mathcal{T}_{\mathfrak{e}}$ of Definition~\ref{def:process T_e} and $\mathcal{T}_{\mathfrak{g}}$ of Definition~\ref{def:process T_g}. 
    For Two-Rings,  MNIST, CIFAR10, and SVHN (resp. 20news),
    the candidates, $K_0 = 5, 10, 15, 50$ (resp. $K_0 = 10, 50, 100, 150$), are examined. The results are shown in Table~\ref{table: robustness k0}.
    
    \item[2)] The hyper-parameter $\alpha$ is used in the critic function of Eq.\eqref{eqn:alpha-tau-model}. The candidates, $\alpha=0,1,2$, are examined. The results are shown in Table~\ref{table: robustness alpha}.
    
    \item[3)]  The importance weight, $\gamma$, is used for the symmetric InfoNCE in MIST objective Eq.\eqref{eq:rewritten implementing unconstrained objective}.  
    The candidates for real-world datasets (resp. synthetic dataset) are 
    $\gamma=0.1,0.5,1.0$ (resp. $\gamma=1,5,10$). The results are shown in Table~\ref{table: robustness gamma}.
\end{enumerate}
Other hyper-parameters are the same as 
those used in MIST of Table~\ref{tb:results of clustering accuracy}. Details are shown in Table~\ref{tb:selected hyperparameters} of Appendix~\ref{subsec:Hyperparameter Setting with Proposed Method}. 

Firstly, as we can see that for most datasets in Table~\ref{table: robustness k0}, MIST is robust to the change of $K_0$. 
The exception is Two-Rings. The clustering accuracy of MIST with $K_0 = 50$ is much lower than  that with $K_0 = 15$. 
A possible reason is that the K-NN graph with a large $K_0$ has edges connecting two data points belonging to different rings. Therefore, maximization of the symmetric InfoNCE based on such a K-NN graph can negatively affect the clustering performance. 
Secondly, Table~\ref{table: robustness alpha} indicates that for all real-world datasets, MIST is robust to the change of $\alpha$ that controls the intensity of the correlation.   
For Two-Rings, however, the performance of MIST is  sensitive to $\alpha$. 
Finally, Table~\ref{table: robustness gamma} shows that for all   the datasets, MIST is stable to the change of $\gamma$.

\section*{Conclusion}
\label{sec:conclusion-and-future-works}

In this study, 
to achieve the goal described in the end of Section~\ref{subsec:motivation},
we proposed topological invariant constraint, which is based on the symmetric InfoNCE, in Section~\ref{subsec:symmetric infonce}. 
Then,  the theoretical advantages are intensively discussed from a deep clustering point of view in Section~\ref{subsubsec:theoretical properties}. In numerical experiments of Section~\ref{sec:numerical-experiments}, 
the efficiency of topologically invariant constraint is confirmed, using MIST defined by combining the constraint and IMSAT.

Future work will refine the symmetric InfoNCE to have fewer hyper-parameters for better and more robust generalization across datasets.
Also, it is worthwhile to investigate a more advanced transformation function to deal with high-dimensional datasets with complex topology. 
Furthermore, developing an efficient way of incorporating information than the MI will enhance the reliability and prediction performance of deep clustering methods.

\subsection*{Acknowledgments}
This work was supported by Japan Society for the Promotion of Science under KAKENHI Grant Number 17H00764, 19H04071, and 20H00576.

\section*{Appendix}

\appendix
\renewcommand*{\theHsection}{\thesection}
\renewcommand*{\theHsubsection}{\thesubsection}

\section{Review of Related Works}
\label{append:review of related works}

\subsection{Deep Clustering Methods without Number of Clusters}
\label{subappend:Review of Deep Clustering Methods without Number of Clusters}
Except for \textsf{Scenario1} and \textsf{Scenario2} where the number of clusters is given, some authors assume that the number of clusters is not given~\citep{DLNC,yang2016joint,DCULVF,mautz2019deep,9171232}.
For example, in DLNC~\citep{DLNC}, for a given unlabeled dataset, 
the feature is extracted by a deep belief network.
Then, the obtained feature vectors are clustered by NMMC (Nonparametric Maximum Margin Clustering) with the estimated number of clusters. 
In DeepCluster~\citep{DCULVF}, 
starting from an excessive number of clusters, the appropriate number of clusters is estimated.

\subsection{Invariant Information Clustering}
\label{subsec:Invariant Information Clustering}
Given image data $\mathcal{D}=\{x_i\}_{i=1}^n$ and the number of clusters $C$, 
IIC (Invariant Information Clustering)~\citep{ji2019invariant} returns the estimated cluster labels $\{\hat{y}_i\}_{i=1}^n$ using the trained model for clustering. 
The training criterion is based on the maximization of the MI between the cluster label of a raw image and the cluster label of the transformed raw image. 
IIC employs the clustering model of $g_{\theta}(x)$ (see Definition~\ref{def:g_theta}), where a CNN is used so at to take advantages of image-specific prior knowledge.

To be more precise, to learn the parameter $\theta$ of the model, IIC maximizes the MI, $I(Y;Y')$, between random variables $Y$ and $Y'$ that take an element in $\{1,\cdots,C\}$. 
Here, $Y$ denotes the random variable of the cluster label with raw image $X \in \mathcal{X}$.
Let $T:\mathcal{X} \to \mathcal{X}$ be an image-specific transformation function, and then
$Y'$ denotes the random variable of the cluster label for the transformed raw image; $T(X)$.
In IIC, the conditional probability $p(y|x)$ is modeled by $g_{\theta}(x)$. 
During the SGD-based optimization stage, given a mini-batch $\mathcal{B} \subseteq \mathcal{D}$,  
$I(Y;Y')$ is computed as follows: 
\begin{enumerate}
    \item[1)] %
        Define $p\left(y,y'|x, T(x)\right) = g^y_{\theta}(x)g^{y^\prime}_{\theta}(T(x))$, 
        where $y$ and $y'$ are the cluster labels of $x$ and $T(x)$,  respectively.
    \item[2)] Compute $p(y,y')=\frac{1}{|\mathcal{B}|}\sum_{x_i \in \mathcal{B}} g^{y}_{\theta}(x_i)g^{y^\prime}_{\theta}\left(T(x_i)\right)$.
    \item[3)] Define $\bar{p}(y,y')$ as the symmetrized probability 
        $(p(y,y')+p(y',y))/2$. 
    \item[4)] Compute the MI $I(Y;Y')$ from $\bar{p}(y,y')$. 
\end{enumerate}
Then, the parameter $\theta$ of the model is found by maximizing 
$I(Y;Y')$ w.r.t. $\theta$. Note that an appropriate transformation $T$ is obtained using image-specific knowledge, such as scaling, skewing, rotation, flipping, etc.

\subsection{Information Maximization for Self-Augmented Training}
\label{subsec:imsat}
In this section, we introduce IMSAT (Information Maximization for Self-Augmented Training)~\citep{hu2017}. To do so, in Appendix~\ref{subsubsec:vat}, firstly we introduce VAT~\citep{miyato2018virtual}, which is an essential regularizer for IMSAT. Then, we explain the objective of IMSAT in Appendix~\ref{subsubsec:Objective of IMSAT}.

\subsubsection{Virtual Adversarial Training}
\label{subsubsec:vat}

Virtual Adversarial Training is a regularizer forcing the smoothness on a given model in the following sense: 
\begin{equation}
\label{eq:smoothness assumption}
     x_i \approx x_j   \Rightarrow \forall y\in\{1,\cdots,C\};\; g_{\theta}^y(x_i) \approx g_{\theta}^y(x_j),
\end{equation}
where $g_\theta$ is defined by Definition~\ref{def:g_theta}.
It should be emphasized that we can train with VAT without labels. 
Let $D_{KL}\left(p_1\|p_2\right)$ denote the KL  (Kullback–Leibler) divergence~\citep{cover1999elements}  between two probability vectors $p_1 \in \Delta^{C-1}$ and $p_2 \in \Delta^{C-1}$. 
During the SGD based optimization stage, 
given a mini-batch $\mathcal{B} \subseteq \mathcal{D}$, the VAT loss, $R_{\rm vat}(\mathcal{B};\theta)$, is defined as,
\begin{equation}
\label{eq: vat-loss}
    R_{\rm vat}(\mathcal{B};\theta) = \frac{1}{|\mathcal{B}|}\sum_{x_i \in \mathcal{B}} D_{KL}\left(
    g_{\theta_l}(x_i)
    \|
    g_{\theta}(x_i + r_i^{\rm adv})
    \right),
\end{equation}
where $r_i^{\rm adv} = \arg\max_{\|r\|_2\leq\epsilon_i}D_{KL}\left(g_{\theta_l}(x_i)\|g_{\theta_l}(x_i + r)\right)$, and $\theta_l$ is the parameter obtained at the $l$-th update.
The radius $\epsilon_i$ depends on $x_i$, and in practice it is estimated via K-NN graph on $\mathcal{D}$; see~\citet{hu2017} for details. 

The approximated $r_i^{\rm adv}$ can be computed by the following three steps;
\begin{enumerate}
    \item[1)] Generate a random unit-vector $u \in \mathbb{R}^d$,
    \item[2)] Compute $v_i = \nabla_r 
    D_{KL}\left(g_{\theta_l}(x_i)\|g_{\theta_l}(x_i + r)\right)
    |_{r=\xi u}$ using the back-propagation,
    \item[3)] $r_i^{\rm adv} = \epsilon_i v_i / \|v_i\|_2$,
\end{enumerate}
where $\xi>0$ is a small positive value.

\subsubsection{Objective of IMSAT}
\label{subsubsec:Objective of IMSAT}
Given $\mathcal{D}=\{x_i\}_{i=1}^n$ and the number of clusters~$C$, IMSAT  provides estimated cluster labels, $\{\hat{y}_i\}_{i=1}^n,\, \hat{y}_i\in\{1,\cdots,C\}$, using $g_{\theta}(x)$ of Definition~\ref{def:g_theta} (statistical model for clustering). 
In IMSAT,  $g_{\theta}(x)$ 
is the simple MLP with the structure $d-1200-1200-C$.  
Using the trained model $g_{\theta^\ast}(x)$, 
we have 
$\hat{y}_i = {\rm argmax}_{y\in\{1,\cdots,C\}}g_{\theta^\ast}^y(x_i)$. 

As for training criterion of the parameter $\theta$, 
IMSAT maximizes the MI, $I(X;Y)$, with the VAT regularization. %
In order to compute $I(X ; Y)$, we assume the following two assumptions: 
1) the conditional probability $p(y|x)$ is modeled by $g_{\theta}(x)$, and 2) the marginal probability $p(x)$ is approximated by the uniform distribution on $\mathcal{D}$.
Thereafter, $I(X;Y)$ is decomposed into 
$I(X ; Y)= H(Y) - H(Y|X)$.
Here $H(Y)$ is Shannon entropy
and $H(Y|X)$ is the conditional entropy~\citep{cover1999elements}.
During the SGD-based optimization, 
given a mini-batch $\mathcal{B} \subseteq \mathcal{D}$, $H(Y)$ and  $H(Y|X)$ are respectively computed as follows:
\begin{equation}
\label{eq:imasat of conditional and mariginal with y}
    -\sum_{y=1}^C p_{\theta}(y)\log p_{\theta}(y)
    \;{\rm and}\;
    -\frac{1}{|\mathcal{B}|}\sum_{x_i \in \mathcal{B}}\sum_{y=1}^C g_{\theta}^y(x_i)
    \log g_{\theta}^y(x_i),
\end{equation}
where $p_{\theta}(y)$ is the approximate marginal probability, 
$\frac{1}{|\mathcal{B}|}\sum_{x_i \in \mathcal{B}}g_{\theta}^y(x_i)$.
The parameter $\theta$ of the model is found by solving the following minimization problem, 
\begin{equation}
\label{eq:practical imsat objective}
     {\rm min}_{\theta}\left\{R_{\rm vat}(\mathcal{B};\theta) - \mu \left( \eta H(Y) - H(Y|X) \right)\right\},
\end{equation}
where $\mu$ and $\eta$ are positive hyper-parameters.

\section{Proofs for Section~\ref{sec:proposed-constraint}}
\label{append:proof for section 3}

\subsection{Proof of Proposition~\ref{prop: relation between mi and symmetric infonce and optimal critic}}
\label{append: proof of relation between mi and symmetric infonce and optimal critic}
From the definition of the MI, $I(Z;Z') = I(Z';Z)$ holds. In addition, we have $I(Z;Z') \geq I_{\mathrm{nce},q}(Z;Z')$ and $I(Z';Z) \geq I_{\mathrm{nce},q}(Z';Z)$ for any function $q$. Therefore, the following inequality holds:
\begin{equation*}
    \begin{split}
        &I(Z;Z') = \frac{I(Z;Z') + I(Z';Z)}{2} \\
        &\geq \frac{I_{\mathrm{nce},q}(Z;Z') + I_{\mathrm{nce},q}(Z';Z)}{2}.
    \end{split}
\end{equation*}

Next, check the optimality. In order to do so, let us review the following inequality~\citep{pmlr-v97-poole19a}: %
\begin{align*}
 I_{\mathrm{nce},q}(Z;Z') 
 &=
 \mathbb{E}\left[ \log\frac{p(Z') e^{q(Z,Z')}}{\mathbb{E}_{Z'}[e^{q(Z,Z')}]} -\log p(Z')\right] \\ 
 &= 
 \mathbb{E}\left[ \log\frac{p(Z') e^{q(Z,Z')}}{\mathbb{E}_{Z'}[e^{q(Z,Z')}]}\right]+H(Z')\\
 &\leq 
 \mathbb{E}\left[\log p(Z'|Z)\right]+H(Z') \\
 &=
 \mathbb{E}_{p(Z,Z')}\left[\log \frac{p(Z,Z')}{p(Z)p(Z')}\right].
\end{align*}
The last inequality comes from the non-negativity of the KL-divergence. 
Therefore, for any $q$, InfoNCE provides a lower bound of $I(Z;Z')$. 
The equality holds if 
\begin{equation}
    \label{eq:optimal critic for infonce}
    q(z,z')= \log p(z|z') + [\text{function of $z$}].
\end{equation}
Thus, if $q$ satisfies $q(z,z')=q(z', z)$, i.e., $\log p(z|z')+h_0(z) = \log p(z'|z)+h_1(z')$ for some function $h_0(z)$ and $h_1(z')$, then the equality between the symmetric InfoNCE and $I(Z;Z')$ holds. As a result, the critic $q$, which is defined as $q(z,z')=\log{\frac{p(z,z')}{p(z)p(z')}}+c,\,c\in\mathbb{R}$, is the optimal critic.

\subsection{Proof of Proposition~\ref{prop:proposed vs iic}}
\label{append: proof of iic vs proposed}
Let us introduce data processing inequality.
Suppose that the random variables $X, Z$ are conditionally independent for a given $Y$. This situation is expressed by 
\begin{align*}
    X \leftrightarrow Y \leftrightarrow Z. 
\end{align*}
Under the above assumption, the data processing inequality
$I(X;Y) \geq I(X;Z)$
holds for the MI. In our formulation, 
the pair of random variables, $X$ and $X'$, is transformed 
to the conditional probabilities,  $p(\cdot|X)=g_{\theta}(X)$ and $p(\cdot|X')=g_{\theta}(X')$, on the $C$-dimensional simplex $\Delta^{C-1}$. 
Then, the cluster label $Y$ (resp. $Y'$) is assumed to be generated from $p(\cdot|X)$ (resp. $p(\cdot|X')$). 
This data generation process satisfies the following relationship: 
\begin{align*}
Y   \leftrightarrow  
p(\cdot \mid X)  
\leftrightarrow    
(X, X^{\prime})  
 \leftrightarrow  p(\cdot \mid X^{\prime}) 
 \leftrightarrow Y^{\prime}. 
\end{align*}
Therefore, for $X'=T(X)$, the data processing inequality leads to 
\begin{align*}
I(Y;Y')
\leq
I(p(\cdot \mid X); p(\cdot \mid X^{\prime})) = I(g_{\theta}(X);g_{\theta}(T(X))). 
\end{align*}

\subsection{Estimation Error of the Symmetric InfoNCE}
\label{append:Estimation Error of InfoNCE}
The symmetric InfoNCE provides an approximation of MI. Here, let us theoretically investigate the estimation error rate of the symmetric InfoNCE with a learnable critic.

Suppose we have training data $x_1,\ldots,x_n$ and their perturbation, $x_i':=t_i(x_i),\,i\in[n]:=\{1,\ldots,n\}$, 
where $t_i$ is a randomly generated map. We assume that $t_1,\ldots,t_n$ are i.i.d.
Recall that the empirical approximation of the InfoNCE loss $I_{\mathrm{nce},q}$ is given by
\begin{align*}
 \widehat{I}_{\mathrm{nce},q}(\theta)
 = \frac{1}{n}\sum_{i}q(g_\theta(x_i),g_\theta(x_i')) 
 - \frac{1}{n}\sum_{i}\log\bigg(\frac{1}{n}\sum_{j}e^{q(g_\theta(x_i),g_\theta(x_j'))}\bigg). 
\end{align*}
The symmetric InfoNCE is defined as $(I_{\textup{nce},q}+I_{\textup{nce},q}')/2$ and its empirical approximation is
\begin{align*}
 &\frac{\widehat{I}_{\mathrm{nce},q}(\theta)+\widehat{I}_{\mathrm{nce},q}'(\theta)}{2} \\
 &= \frac{1}{n}\sum_{i}q(g_\theta(x_i),g_\theta(x_i'))
 - \frac{1}{2n}\Bigg(\sum_{i}\log\bigg(\frac{1}{n}\sum_{j}e^{q(g_\theta(x_i),g_\theta(x_j'))}\bigg) \\ 
 &\;\;\;\;\;\;\;\;\;\;\;\;\;\;\;\;\;\;\;\;\;\;\;\;\;\;\;\;\;\;\;\;\;\;\;\;\;\;\;\;\;\;\;\;\;\;\;\;\;\;\;\;\;\;\;\;\;\;\;\;\;+\sum_{i}\log\bigg(\frac{1}{n}\sum_{j}e^{q(g_\theta(x_i'),g_\theta(x_j))}\bigg)\Bigg).
\end{align*}
Let $I_{\mathrm{sym\_nce},q}$ and $\widehat{I}_{\mathrm{sym\_nce},q}$ denote the symmetric InfoNCE and the empirical approximation,  respectively.
Let $\mathcal{Q}$ be a set of critics. The MI is approximated by
\begin{align*}
I_{\mathcal{Q}}(\theta)=\sup_{q\in\mathcal{Q}}I_{\mathrm{sym\_nce},q}(\theta). 
\end{align*}
The empirical approximation of $I_{\mathcal{Q}}(\theta)$
is given by 
$\widehat{I}_{\mathcal{Q}}(\theta)=\sup_{q\in\mathcal{Q}}\widehat{I}_{\mathrm{sym\_nce},q}(\theta)$. 
Then, the parameter $\widehat{\theta}$ of the model is given by the maximizer of $\widehat{I}_{\mathcal{Q}}(\theta)$, i.e.,
\begin{align*}
 \max_{\theta\in\Theta}\widehat{I}_{\mathcal{Q}}(\theta)\ \longrightarrow\ \widehat{\theta}. 
\end{align*}
Let $I(\theta)$ be the mutual information between $g_\theta(X)$ and $g_\theta(X')$. 
The maximizer of $I(\theta)$ (resp. $I_{\mathcal{Q}}(\theta)$) 
is denoted by $\theta^*$ (resp. $\theta_{\mathcal{Q}}\in\Theta$). 

We evaluate the mutual information at $\widehat{\theta}$, i.e., $I(\widehat{\theta})$. 
From the definition, we have 
\begin{align}
\label{eqn:eval-MI}
 0\leq I(\theta^*)  -  I(\widehat{\theta}) 
& \leq 
 I(\theta^*)  -  I_{\mathcal{Q}}(\widehat{\theta})  
 \leq 
 \underbrace{I(\theta^*)  -   I_{\mathcal{Q}}(\theta_{\mathcal{Q}})}_{\text{approximation error}\geq0}
 + 
 \underbrace{I_{\mathcal{Q}}(\theta_{\mathcal{Q}})- I_{\mathcal{Q}}(\widehat{\theta})}_{\text{estimation error}\geq0}. 
\end{align}
We consider the estimation error bound. The optimality of $\widehat{\theta}$ leads to 
\begin{align}
 0&\leq I_{\mathcal{Q}}(\theta_{\mathcal{Q}})- I_{\mathcal{Q}}(\widehat{\theta})
 \leq 
 I_{\mathcal{Q}}(\theta_{\mathcal{Q}})
 - \widehat{I}_{\mathcal{Q}}(\theta_{\mathcal{Q}})
 +\widehat{I}_{\mathcal{Q}}(\theta_{\mathcal{Q}})
 -\widehat{I}_{\mathcal{Q}}(\widehat{\theta})
 +\widehat{I}_{\mathcal{Q}}(\widehat{\theta})
 -I_{\mathcal{Q}}(\widehat{\theta})  \nonumber\\
\label{eqn:sup-BD}
& \leq 
 I_{\mathcal{Q}}(\theta_{\mathcal{Q}})
 - \widehat{I}_{\mathcal{Q}}(\theta_{\mathcal{Q}})
 +\widehat{I}_{\mathcal{Q}}(\widehat{\theta})
 -I_{\mathcal{Q}}(\widehat{\theta})
 \leq 2\sup_{\theta\in\Theta}|I_{\mathcal{Q}}(\theta) - \widehat{I}_{\mathcal{Q}}(\theta)|. 
\end{align}
Let us evaluate the worst-case gap between $I_{\mathcal{Q}}(\theta)$ and
$\widehat{I}_{\mathcal{Q}}(\theta)$: 
\begin{align*}
 I_{\mathcal{Q}}(\theta) - \widehat{I}_{\mathcal{Q}}(\theta)
 &=
 \sup_{q}\inf_{q'}I_{\mathrm{sym\_nce},q}(\theta) - \widehat{I}_{\mathrm{sym\_nce},q'}(\theta) \\
 &\leq 
 \sup_{q}I_{\mathrm{sym\_nce},q}(\theta) - \widehat{I}_{\mathrm{sym\_nce},q}(\theta). 
\end{align*}
Likewise, we have 
$\widehat{I}_{\mathcal{Q}}(\theta)- I_{\mathcal{Q}}(\theta)\leq  \sup_{q}\widehat{I}_{\mathrm{sym\_nce},q}(\theta)-I_{\mathrm{sym\_nce},q}(\theta)$. 
Therefore, 
\begin{align*}
 \sup_{\theta\in\Theta}|I_{\mathcal{Q}}(\theta) - \widehat{I}_{\mathcal{Q}}(\theta)|
 \leq 
 \sup_{\theta\in\Theta,q\in\mathcal{Q}}|I_{\mathrm{sym\_nce},q}(\theta) - \widehat{I}_{\mathrm{sym\_nce},q}(\theta)|. 
\end{align*}
To derive the convergence rate, 
we use the Uniform Law of Large Numbers (ULLN) \citep{mohri18:_found_machin_learn} to the following function classes, 
\begin{align*}
 \mathcal{G}&=\{(x,t)\mapsto q(g_\theta(x),g_\theta(t(x)))\,:\,\theta\in\Theta,\ q\in\mathcal{Q}\}, \\
 \exp\circ\mathcal{G}&=\{(x,t)\mapsto \exp\{q(r, g_\theta(t(x)))\}\,:\,\theta\in\Theta,\ q\in\mathcal{Q},\, r\in\Delta^{C-1}\}. 
\end{align*}
Suppose that the model $(g_{\theta}^{y})_{y\in[C]}$ with any permutation of cluster label is realized by the other parameter $\theta'$. 
For instance, when $C=2$, for any $\theta$ there exists $\theta'$ such that
$(g_{\theta}^2, g_{\theta}^1)=(g_{\theta'}^1, g_{\theta'}^2)$ holds. 
Then, let us define the following function class $\mathcal{N}$ 
by
\begin{align*}
\mathcal{N}=\{x\mapsto g_{\theta,1}(x)\,:\,\theta\in\Theta\}, 
\end{align*}
where $g_{\theta,1}$ is the first element of $g_{\theta}$. 
We evaluate the estimation error bound in terms of 
the Rademacher complexity of $\mathcal{N}$. 
See \citet{bartlett02:_radem_gauss_compl,mohri18:_found_machin_learn}
for details of Rademacher complexity.

We assume the following conditions: 
\begin{enumerate}
 \item[(A)] Any $q(r,r')$ in $\mathcal{Q}$ 
	    is expressed as $q(r,r')=\phi_q(r^{\top}r')$ for $r,r'\in\Delta^{C-1}$, 
	    where $\phi_q:[0,1]\rightarrow[a,b]$. 
	    We assume that the range of $\phi_q$ is uniformly bounded in the interval $[a,b]$. 
 \item[(B)] The Lipschitz constant $\|\phi_q\|_{\mathrm{Lip}}$ of $\phi_q$ is uniformly bounded, i.e., 
	    $$\sup_{q\in\mathcal{Q}}\|\phi_q\|_{\mathrm{Lip}}\leq L<\infty.$$
\end{enumerate}
We consider the Rademacher complexity of $\mathcal{G}$ and $\exp\circ\mathcal{G}$. 
Let $\sigma_i, i\in[n]$ be i.i.d. Rademacher random variables. 
Given $D=\{(x_i,x_i'),\,i\in[n]\}$, the empirical Rademacher complexity is 
\begin{align*}
 \widehat{\mathfrak{R}}_D(\mathcal{G})
 &=
 \mathbb{E}_{\sigma}\bigg[\sup_{\theta\in\Theta,q\in\mathcal{Q}}
\frac{1}{n} \sum_{i}\sigma_iq(g_\theta(x_i),g_\theta(x_i))\bigg] \\
 &\leq 
 \mathbb{E}_{\sigma}\bigg[ \sup_{\theta\in\Theta,q\in\mathcal{Q},r\in\Delta^{C-1}}
 \frac{1}{n}\sum_{i}\sigma_iq(r,g_\theta(x_i))\bigg],\\
 \widehat{\mathfrak{R}}_D(\exp\circ\mathcal{G})
 &=
 \mathbb{E}_{\sigma}\bigg[\sup_{\theta\in\Theta,q\in\mathcal{Q},r\in\Delta^{C-1}}
 \frac{1}{n}\sum_{i}\sigma_i\exp\{q(r,g_\theta(x_i))\}\bigg] \\
 &\leq
 e^{b} \mathbb{E}_{\sigma}\bigg[\sup_{\theta\in\Theta,q\in\mathcal{Q},r\in\Delta^{C-1}}
\frac{1}{n}\sum_{i}\sigma_iq(r,g_\theta(x_i))\bigg]. 
\end{align*}
The inequality in the second line is obtained by Talagrand's 
lemma~\citep{mohri18:_found_machin_learn}. 
Due to the assumption on the function $q(r,r')$, 
for $g_\theta = (g_{\theta,1},\ldots,g_{\theta,C})\in\Delta^{C-1}$ we have 
\begin{align*}
 \mathbb{E}_{\sigma}\bigg[\sup_{\theta,q,r}\frac{1}{n}\sum_{i}\sigma_iq(r,g_\theta(x_i))\bigg]
&=
 \mathbb{E}_{\sigma}\bigg[\sup_{\theta,q,r}\frac{1}{n}\sum_{i}\sigma_i\phi_q(r^\top g_\theta(x_i))\bigg] \\
&\leq
 L\mathbb{E}_{\sigma}\bigg[\sup_{\theta,r}\frac{1}{n}\sum_{i}\sigma_i r^\top g_\theta(x_i)\bigg]\\
&=
 L\mathbb{E}_{\sigma}\bigg[\sup_{\theta}\max_{c\in[C]}\frac{1}{n}\sum_{i}\sigma_i g_{\theta,c}(x_i)\bigg] \\
 &=
 L\mathbb{E}_{\sigma}\bigg[\sup_{\theta}\frac{1}{n}\sum_{i}\sigma_i g_{\theta,1}(x_i)\bigg].
\end{align*}
In the last inequality, again Talagrand's lemma is used. 
Note that since we deal with a general case in which the probability distribution of $x_i$ and $x_i'$ may not be equal to each other, it is worth considering a counterpart w.r.t. the probability distribution of $x_i$, i.e., we have
\begin{align*}
 \widehat{\mathfrak{R}}_D(\mathcal{G})
 &=
 \mathbb{E}_{\sigma}\bigg[\sup_{\theta\in\Theta,q\in\mathcal{Q}}
\frac{1}{n} \sum_{i}\sigma_iq(g_\theta(x_i),g_\theta(x_i'))\bigg] \\
 &\leq 
 \mathbb{E}_{\sigma}\bigg[ \sup_{\theta\in\Theta,q\in\mathcal{Q},r\in\Delta^{C-1}}
 \frac{1}{n}\sum_{i}\sigma_iq(r,g_\theta(x_i'))\bigg],\\
 \widehat{\mathfrak{R}}'_D(\exp\circ\mathcal{G})
 &=
 \mathbb{E}_{\sigma}\bigg[\sup_{\theta\in\Theta,q\in\mathcal{Q},r\in\Delta^{C-1}}
 \frac{1}{n}\sum_{i}\sigma_i\exp\{q(r,g_\theta(x_i'))\}\bigg] \\
 &\leq
 e^{b} \mathbb{E}_{\sigma}\bigg[\sup_{\theta\in\Theta,q\in\mathcal{Q},r\in\Delta^{C-1}}
\frac{1}{n}\sum_{i}\sigma_iq(r,g_\theta(x_i'))\bigg],
\end{align*}
and,
\begin{align*}
 \mathbb{E}_{\sigma}\bigg[\sup_{\theta,q,r}\frac{1}{n}\sum_{i}\sigma_iq(r,g_\theta(x_i'))\bigg]
&=
 \mathbb{E}_{\sigma}\bigg[\sup_{\theta,q,r}\frac{1}{n}\sum_{i}\sigma_i\phi_q(r^\top g_\theta(x_i'))\bigg] \\
&\leq
 L\mathbb{E}_{\sigma}\bigg[\sup_{\theta,r}\frac{1}{n}\sum_{i}\sigma_i r^\top g_\theta(x_i')\bigg]\\
&=
 L\mathbb{E}_{\sigma}\bigg[\sup_{\theta}\max_{c\in[C]}\frac{1}{n}\sum_{i}\sigma_i g_{\theta,c}(x_i')\bigg] \\
 &=
 L\mathbb{E}_{\sigma}\bigg[\sup_{\theta}\frac{1}{n}\sum_{i}\sigma_i g_{\theta,1}(x_i')\bigg].
\end{align*}
For our purpose it is sufficient to find the Rademacher complexity $\mathfrak{R}_n(\mathcal{N})$ (resp. $\mathfrak{R}_n'(\mathcal{N})$) of $\mathcal{N}$ w.r.t. the probability distribution of $x_i$ (resp. $x_i'$). 
For the standard neural network models, both the Rademacher complexities $\mathfrak{R}_n(\mathcal{N})$ and $\mathfrak{R}_n'(\mathcal{N})$ are of the order of $n^{-1/2}$
and the coefficient depends on the maximum norm of the weight~\citep{shalev2014understanding}. 
From the above calculation, %
we have 
\begin{align*}
\mathfrak{R}_n(\mathcal{G})\leq c\, \mathfrak{R}_n(\mathcal{N}),\quad  \mathfrak{R}_n(\exp\circ\mathcal{G})\leq c\, \mathfrak{R}_n(\mathcal{N}), 
\end{align*}
where $c$ is a positive constant depending on $b$ and $L$.
Note that the same argument holds for $\mathfrak{R}_n'(\mathcal{N})$.
In the below, $c$ is a positive constant that can be different line by line. 
Furthermore, let us evaluate the Rademacher complexity of the function set
\begin{align*}
 x\ \longmapsto\ \log\mathbb{E}_{X'}\big[e^{q(g_\theta(x)^\top g_{\theta}(X'))}\big]. 
\end{align*}
We use the upper bound of $\mathfrak{R}_n(\exp\circ\mathcal{G})$. 
The logarithmic function is Lipschitz continuous on the bounded interval $[e^a, e^b]$ and Lipschitz constant is bounded above by 
$e^{-a}$ on the interval. The empirical Rademacher complexity is given by 
\begin{align}
&\mathbb{E}_\sigma
 \bigg[ \sup_{\theta,q}\frac{1}{n}\sum_{i=1}^{n}\sigma_i \log\mathbb{E}_{X'}\bigg[e^{q(g_\theta(x_i)^\top g_{\theta}(X'))}\bigg] \bigg] \nonumber\\
& \leq 
e^{-a}\mathbb{E}_\sigma
 \bigg[ \mathbb{E}_{X'}\bigg[\sup_{\theta,q} \frac{1}{n}\sum_{i=1}^{n}\sigma_i e^{q(g_\theta(x_i)^\top g_{\theta}(X'))}\bigg] \bigg]
 \nonumber\\
& \leq 
e^{-a} \mathbb{E}_\sigma \bigg[\sup_{\theta,q,r} \frac{1}{n}\sum_{i=1}^{n}\sigma_i e^{q(g_\theta(x_i)^\top r)}\bigg]  \nonumber\\
\label{eqn:bound_log-exp}
& \leq  
Le^{b-a}\,\mathbb{E}_\sigma\bigg[\sup_{\theta} \frac{1}{n}\sum_{i=1}^{n}\sigma_i g_{\theta,1}(x_i)\bigg].
\end{align}
From the above calculation, the following theorem holds. 
\begin{thm}
 \label{thm:estimation_BD}
 Assume the condition (A) and (B). 
 Let us define $\varepsilon_{\delta,n}^{\mathcal{N}}$ as
 \begin{align*}
  \varepsilon_{\delta,n}^{\mathcal{N}}=\frac{1}{2}\left(\mathfrak{R}_n(\mathcal{N}) + \mathfrak{R}_n'(\mathcal{N}) \right) + \sqrt{\frac{\log(1/\delta)}{n}}, 
 \end{align*}
 where $\mathfrak{R}_n(\mathcal{N})$ (resp. $\mathfrak{R}_n'(\mathcal{N})$) is the Rademacher complexity of $\mathcal{N}$ for $n$ samples following the probability distribution of $x_i$ (resp. the probability distribution of $x_i'$). 
 Then, with the probability greater than $1-\delta$, we have 
\begin{align*}
I_{\mathcal{Q}}(\theta_{\mathcal{Q}})- I_{\mathcal{Q}}(\widehat{\theta}) \leq c\,\varepsilon_{\delta,n}^{\mathcal{N}}, 
\end{align*}
 where $c$ is a positive constant depending on $a,b$ and $L$. 
\end{thm}

\begin{proof}
The proof of Theorem~\ref{thm:estimation_BD} is the following. 
From the definition of the symmetric InfoNCE, we have 
\begin{align*}
&\sup_{\theta}|I_{\mathcal{Q}}(\theta)-\widehat{I}_{\mathcal{Q}}(\theta)| \\
&\leq 
 \sup_{\theta,q}|I_{\mathrm{sym\_nce},q}(\theta)-\widehat{I}_{\mathrm{sym\_nce},q}(\theta)|\\
&\leq 
 \sup_{\theta,q}
\bigg|\frac{1}{n} \sum_{i=1}^{n}q(g_\theta(x_i),g_\theta(x_i')) - \mathbb{E}[q(g_\theta(X),g_\theta(X'))]\bigg|\\
&\;\;\;\;\;\;\;\;+
 \frac{1}{2}\sup_{\theta,q} \bigg|
 \frac{1}{n}\sum_{i=1}^{n}\log\frac{1}{n}\sum_{j=1}^{n}e^{q(g_{\theta}(x_i),g_\theta(x_j'))}
 -
 \mathbb{E}_{X}\log\mathbb{E}_{X'}\big[e^{q(g_\theta(X)^\top g_{\theta}(X'))}\big]\bigg|\\
&\;\;\;\;\;\;\;\;\;\;\;\;+
 \frac{1}{2}\sup_{\theta,q} \bigg|
 \frac{1}{n}\sum_{i=1}^{n}\log\frac{1}{n}\sum_{j=1}^{n}e^{q(g_{\theta}(x_i'),g_\theta(x_j))}
 -
 \mathbb{E}_{X'}\log\mathbb{E}_{X}\big[e^{q(g_\theta(X')^\top g_{\theta}(X))}\big]\bigg|. 
\end{align*}
From the Rademacher complexity of $\mathfrak{R}_n(\mathcal{G})$, the first term in the above is bounded above by 
$\varepsilon_{n,\delta}^{\mathcal{N}}$ up to a positive constant. 
Next, 
let us define the following:
\begin{align*}
    \varepsilon_{n,\delta}^{\mathcal{N},1}=\mathfrak{R}_n(\mathcal{N})+\sqrt{\frac{\log{(1/\delta)}}{n}}
    \quad
    \varepsilon_{n,\delta}^{\mathcal{N},2}= \mathfrak{R}_n'(\mathcal{N})+\sqrt{\frac{\log{(1/\delta)}}{n}}.
\end{align*}
It is clear that $2\varepsilon_{n,\delta}^{\mathcal{N}}=\varepsilon_{n,\delta}^{\mathcal{N},1}+\varepsilon_{n,\delta}^{\mathcal{N},2}$.
Then, the ULLN with the upper bound of Eq.\eqref{eqn:bound_log-exp} leads to the following with the probability greater than $1-\delta/4$ that
\begin{align*}
\sup_{\theta,q} \bigg|
 \frac{1}{n}\sum_{i=1}^{n}\log\mathbb{E}_{X'}\big[e^{q(g_\theta(x_i),g_{\theta}(X'))}\big]
 -
 \mathbb{E}_{X}\log\mathbb{E}_{X'}\big[e^{q(g_\theta(X),g_{\theta}(X'))}\big]
 \bigg| \leq c\,\varepsilon_{n,\delta}^{\mathcal{N},1}. 
\end{align*}
Hence, with the probability greater thatn $1-\delta/2$ we have
\begin{align*}
&\phantom{\leq}
 \bigg|
 \frac{1}{n}\sum_{i=1}^{n}\log\frac{1}{n}\sum_{j=1}^{n}e^{q(g_{\theta}(x_i),g_\theta(x_j'))}
 - \mathbb{E}_{X}\log\mathbb{E}_{X'}\big[e^{q(g_\theta(X)^\top g_{\theta}(X'))}\big]
\bigg|\\
&\leq 
 \bigg|
  \frac{1}{n}\sum_{i=1}^{n}\log\frac{1}{n}\sum_{j=1}^{n}e^{q(g_{\theta}(x_i),g_\theta(x_j'))}
 -\frac{1}{n}\sum_{i=1}^{n}\log\mathbb{E}_{X'}\big[e^{q(g_\theta(x_i), g_{\theta}(X'))}\big]
\bigg|\\
&\;\;\;\;\;\;\;\;\;\;\;\;+\bigg|
\frac{1}{n}\sum_{i=1}^{n}\log\mathbb{E}_{X'}\big[e^{q(g_\theta(x_i)^\top g_{\theta}(X'))}\big]
 -
 \mathbb{E}_{X}\log\mathbb{E}_{X'}\big[e^{q(g_\theta(X),g_{\theta}(X'))}\big]
\bigg|\\
&\leq 
 e^{-a}
 \frac{1}{n}\sum_{i=1}^{n}
 \bigg|
 \frac{1}{n}\sum_{j=1}^{n}e^{q(g_{\theta}(x_i),g_\theta(x_j'))}
 -
 \mathbb{E}_{X'}\big[e^{q(g_\theta(x_i),g_{\theta}(X'))}\big]\bigg|
 +c\,\varepsilon_{n,\delta}^{\mathcal{N},1}\\
&\leq 
 e^{-a}
 \sup_{r\in\Delta^{C-1}}
 \bigg|
 \frac{1}{n}\sum_{j=1}^{n}e^{q(r,g_\theta(x_j'))}
 -
 \mathbb{E}_{X'}\big[e^{q(r, g_{\theta}(X'))}\big]\bigg|
 +c\,\varepsilon_{n,\delta}^{\mathcal{N},1}\\
& \leq
c\,\varepsilon_{n,\delta}^{\mathcal{N}}. 
\end{align*}
Similarly, with the probability greater than $1-\delta/2$ we have
\begin{align*}
 \bigg|
 \frac{1}{n}\sum_{i=1}^{n}\log\frac{1}{n}\sum_{j=1}^{n}e^{q(g_{\theta}(x_i'),g_\theta(x_j))}
 - \mathbb{E}_{X'}\log\mathbb{E}_{X}\big[e^{q(g_\theta(X')^\top g_{\theta}(X))}\big]
 \bigg|
 \leq c\,\varepsilon_{n,\delta}^{\mathcal{N}}.
\end{align*}
Eventually, the worst-case error $\sup_{\theta}|I_{\mathcal{Q}}(\theta)-\widehat{I}_{\mathcal{Q}}(\theta)|$
is bounded above by $\varepsilon_{n,\delta}^{\mathcal{N}}$ up to a positive factor. 
The above bound with inequalities in Eq.\eqref{eqn:eval-MI} and \eqref{eqn:sup-BD} lead to the conclusion. \qed
\end{proof}

We then show that the critic functions defined in Eq.\eqref{eqn:alpha-tau-model} satisfy both the condition (A) and (B).
Recall the definition of the critic functions:
\begin{align*}
    q(z, z') = 
    \log\left(\exp_{\alpha}\big(\tau(z^\top z'-1)\big)\right),\quad\mathrm{where}\quad \alpha\in\mathbb{R},\quad \tau\geq 0.
\end{align*}

\begin{lem}
 \label{lem:properties_of_critics}
 Given real values $l\geq 0$, $w> 0$, $0<d<1$ and $s<0$, define $\Xi=\{(\alpha,\tau)\in\mathbb{R}\times\mathbb{R}_{\geq 0}\,:\, \tau\leq l,\,w\leq |\alpha-1|,\,s\leq (1-\alpha)\tau\leq 1-d\}\cup\{(1,\tau)\,:\,0\leq \tau\leq l\}$, $\phi_{(\alpha,\tau)}=\log\left(\exp_{\alpha}\big(\tau(z^\top z'-1)\big)\right)$ and $\mathcal{Q}=\{q:\Delta^{C-1}\times\Delta^{C-1}\to\mathbb{R}\,:\,\exists (\alpha,\tau)\in\Xi\,\,\mathrm{s.t.}\,q=\phi_{(\alpha,\tau)}\}$.
 Then every $q\in\mathcal{Q}$ satisfies both the condition (A) and (B).
\end{lem}
\begin{proof}
 Let $q\in\mathcal{Q}$.
 From the definition of $\mathcal{Q}$, there exists some $(\alpha,\tau)\in \Xi$ such that $q$ is expressed as $q(r,r')=\phi_{(\alpha,\tau)}(r^{\top}r')$ for any $r,r'\in\Delta^{C-1}$.
 Moreover, $\phi_{(\alpha,\tau)}$ is uniformly bounded in the following closed interval
 $$[\min\{\log{d}^{1/(1-\alpha)},\log{(1-s)}^{1/(1-\alpha)}\},\max\{\log{d}^{1/(1-\alpha)},\log{(1-s)}^{1/(1-\alpha)}\}],$$
 when $\alpha\neq 1$, and in $[-l,0]$ when $\alpha=1$.
 Therefore, $q$ satisfies the condition (A).
 Let us show every $q=\phi_{(\alpha,\tau)}\in\mathcal{Q}$ also satisfies the condition (B).
 When $\alpha=1$, from the definition of the function $\exp_{\alpha}$, we have $\phi_{(\alpha,\tau)}(x)=\tau (x-1)$ on $x\in [0,1]$.
 Hence, $\|\phi_{(\alpha,\tau)}\|_{\mathrm{Lip}}\leq \tau\leq l<\infty$.
 When $\alpha\neq 1$, since $\phi_{(\alpha,\tau)}(x)=(1-\alpha)^{-1}\log(1+(1-\alpha)\tau(x-1))$ on $x\in[0,1]$ we have:
 \begin{itemize}
     \item When $0\leq (1-\alpha)\tau\leq 1-d$, we have $\|\phi_{(\alpha,\tau)}\|_{\mathrm{Lip}}\leq \tau/(1-(1-\alpha)\tau)\leq \tau/d<\infty$.
     \item When $s\leq (1-\alpha)\tau<0$, then we have $\|\phi_{(\alpha,\tau)}\|_{\mathrm{Lip}}\leq \tau<\infty$.
 \end{itemize}
 Therefore, there exists a non-negative constant $L$ such that $$\sup_{\phi_{(\alpha,\tau)}\in\mathcal{Q}}\|\phi_{(\alpha,\tau)}\|_{\mathrm{Lip}}\leq L<\infty.$$
 This implies that $q=\phi_{(\alpha,\tau)}$ satisfies the condition (B). \qed
\end{proof}

Now we are ready to show the main result on the statistical analysis in this section.

\begin{thm}[Formal version of Theorem~\ref{thm:Estimation Error Analysis Informal}]
\label{thm:Estimation Error Analysis Formal}   
Let $\mathcal{Q}$ be the set defined in the setting of Lemma~\ref{lem:properties_of_critics}.
Then, with the probability greater than $1-\delta$, we have 
\begin{align*}
I_{\mathcal{Q}}(\theta_{\mathcal{Q}})- I_{\mathcal{Q}}(\widehat{\theta}) \leq c\,\varepsilon_{\delta,n}^{\mathcal{N}}, 
\end{align*}
 where $c$ is a positive constant depending on $a,b$ and $L$. 
As a result, with probability at least $1-\delta$ the gap between $I(\theta^*)$ and $I(\widehat{\theta})$ is given by 
\begin{align*}
 0\leq I(\theta^*)-I(\widehat{\theta})\leq I(\theta^*)-I_{\mathcal{Q}}(\theta_{\mathcal{Q}}) + c\,\varepsilon_{n,\delta}^{\mathcal{N}}. 
\end{align*}
\end{thm}

\begin{proof}%
 From Theorem~\ref{thm:estimation_BD} and Lemma~\ref{lem:properties_of_critics}, we obtain the claim.\qed
\end{proof}

\subsection{Proof of Theorem~\ref{thm: downstream tasks guarantee analysis}} %
\label{append: Downstream Tasks Guarantee}

We show Theorem~\ref{thm: downstream tasks guarantee analysis} based on the results by \citet{wang2022chaos}.
Recall the definition of the critic function Eq.\eqref{eqn:alpha-tau-model}; when $\alpha=1$, the critic function $q(g_{\theta}(X),g_{\theta}(T(X)))$ is just $q(g_{\theta}(X),g_{\theta}(T(X)))=\tau(g_{\theta}(X)^{\top}g_{\theta}(T(X))-1)$.
In this case, the symmetric InfoNCE loss, $-\frac{1}{2}(I_{\mathrm{nce}}+I_{\mathrm{nce}}')$, is written as
\begin{align*}
    &-\mathbb{E}_{p(Z,Z')}[\tau(Z^{\top}Z'-1)] + \frac{1}{2}(\mathbb{E}_{p(Z)}[\log\mathbb{E}_{p(Z')}[\exp(\tau(Z^{\top}Z'-1))] \\
    &\;\;\;\;\;\;\;\;\;\;\;\;\;\;\;\;\;\;\;\;\;\;\;\;\;\;\;\;\;\;\;\;\;\;\;\;\;\;\;\;\;\;\;\;\;\;\;\;\;\;\;\;\;\;\;\;+\mathbb{E}_{p(Z')}[\log\mathbb{E}_{p(Z)}[\exp(\tau(Z^{\top}Z'-1))]]]).
\end{align*}
The following Proposition~\ref{prop: upper bound for downstream tasks guarantee} provides an upper bound of the symmetric mean supervised loss involving the symmetric InfoNCE loss.

\begin{prop}
\label{prop: upper bound for downstream tasks guarantee}
 We have,
 \begin{align*}
     \mathcal{L}_{\mathrm{SCE}}^{\mu,\widetilde{\mu}}(g_\theta)\leq -\frac{1}{2}\left(I_{\mathrm{nce}}+I_{\mathrm{nce}}'\right) +\frac{1}{2}\left(\sqrt{\mathrm{Var}(Z|Y)}+\sqrt{\mathrm{Var}(Z'|Y)}+2\log{C}\right),
 \end{align*}
 where $\mathrm{Var}(Z|Y)=\mathbb{E}_{p(Y)}[\mathbb{E}_{p(Z|Y)}[\|{\tau}Z-\mu_{Y}\|_{\infty}^{2}]]$, $\mathrm{Var}(Z'|Y) =\mathbb{E}_{p(Y)}[\mathbb{E}_{p(Z'|Y)}[\|{\tau}Z'-\mu_{Y}\|_{\infty}^{2}]]$.
\end{prop}
\begin{proof}
The proof of Proposition~\ref{prop: upper bound for downstream tasks guarantee} is mainly due to Theorem~A.3 of \citet{wang2022chaos}, but slightly different because we now focus on the symmetric InfoNCE with the critic function of Eq.\eqref{eqn:alpha-tau-model}.
We show the detail of our proof based on \citet{wang2022chaos} for the sake of completeness.
 \begin{align*}
  &-\frac{1}{2}\left(I_{\mathrm{nce}}+I_{\mathrm{nce}}'\right)\\
  &=-\mathbb{E}_{p(Z,Z')}[{\tau}Z^{\top}Z'-\tau]+\frac{1}{2}(\mathbb{E}_{p(Z)}[\log\mathbb{E}_{p(Z')}[\exp({\tau}Z^{\top}Z'-\tau)]\\
  &\;\;\;\;\;\;\;\;\;\;\;\;\;\;\;\;\;\;\;\;\;\;\;\;\;\;\;\;\;\;\;\;\;\;\;\;\;\;\;\;\;\;\;\;\;\;\;\;\;\;\;\;\;+\mathbb{E}_{p(Z')}[\log\mathbb{E}_{p(Z)}[\exp({\tau}Z^{\top}Z'-\tau)]]])\\
  &=-\mathbb{E}_{p(Z,Z')}[{\tau}Z^{\top}Z'] + \frac{1}{2}(\mathbb{E}_{p(Z)}[\log\mathbb{E}_{p(Z')}[\exp({\tau}Z^{\top}Z')]\\
  &\;\;\;\;\;\;\;\;\;\;\;\;\;\;\;\;\;\;\;\;\;\;\;\;\;\;\;\;\;\;\;\;\;\;\;\;\;\;\;\;\;\;\;\;\;\;\;\;\;\;\;\;\;+\mathbb{E}_{p(Z')}[\log\mathbb{E}_{p(Z)}[\exp({\tau}Z^{\top}Z')]]])\\
  &=-\mathbb{E}_{p(Z,Z')}[{\tau}Z^{\top}Z']+\frac{1}{2}(\mathbb{E}_{p(Z)}[\log\mathbb{E}_{p(Y)}[\mathbb{E}_{p(Z'|Y)}[\exp({\tau}Z^{\top}Z')]]]\\
  &\;\;\;\;\;\;\;\;\;\;\;\;\;\;\;\;\;\;\;\;\;\;\;\;\;\;\;\;\;\;\;\;\;\;\;\;\;\;\;\;\;\;\;\;\;\;\;\;\;\;\;\;\;+\mathbb{E}_{p(Z')}[\log\mathbb{E}_{p(Y)}[\mathbb{E}_{p(Z|Y)}[\exp({\tau}Z^{\top}Z')]]])\\
  &\geq -\mathbb{E}_{p(Z,Z')}[{\tau}Z^{\top}Z']+\frac{1}{2}(\mathbb{E}_{p(Z)}[\log\mathbb{E}_{p(Y)}[\exp(\mathbb{E}_{p(Z'|Y)}[{\tau}Z^{\top}Z'])]]\\
  &\;\;\;\;\;\;\;\;\;\;\;\;\;\;\;\;\;\;\;\;\;\;\;\;\;\;\;\;\;\;\;\;\;\;\;\;\;\;\;\;\;\;\;\;\;\;\;\;\;\;\;\;\;+\mathbb{E}_{p(Z')}[\log\mathbb{E}_{p(Y)}[\exp(\mathbb{E}_{p(Z|Y)}[{\tau}Z^{\top}Z']]]))\\
  &=-\frac{1}{2}\Big(\mathbb{E}_{p(Z,Z',Y)}[Z^{\top}\widetilde{\mu}_{Y}+Z^{\top}({\tau}Z'-\widetilde{\mu}_{Y})]\\
  &\;\;\;\;\;\;\;\;\;\;\;\;\;\;\;\;\;\;\;\;\;\;\;\;\;\;\;\;\;\;\;\;\;\;\;\;\;\;\;\;\;\;\;\;\;\;\;\;\;\;\;\;\;+\mathbb{E}_{p(Z,Z',Y)}[{Z'}^{\top}\mu_{Y}+{Z'}^{\top}({\tau}Z-\mu_{Y})]\Big)\\
  &\quad+\frac{1}{2}\left(\mathbb{E}_{p(Z)}[\log\mathbb{E}_{p(Y)}[\exp(Z^{\top}\widetilde{\mu}_{Y})]]+\mathbb{E}_{p(Z')}[\log\mathbb{E}_{p(Y)}[\exp({Z'}^{\top}\mu_{Y})]]\right)\\
  &=\frac{1}{2}\left(-\mathbb{E}_{p(Z,Z',Y)}[Z^{\top}\widetilde{\mu}_{Y}+Z^{\top}({\tau}Z'-\widetilde{\mu}_{Y})]+\mathbb{E}_{p(Z)}[\log\mathbb{E}_{p(Y)}[\exp(Z^{\top}\widetilde{\mu}_{Y})]]\right)\\
  &\quad+\frac{1}{2}\left(-\mathbb{E}_{p(Z,Z',Y)}[{Z'}^{\top}\mu_{Y}+{Z'}^{\top}({\tau}Z-\mu_{Y})]+\mathbb{E}_{p(Z')}[\log\mathbb{E}_{p(Y)}[\exp({Z'}^{\top}\mu_{Y})]]\right)\\
  &\geq \frac{1}{2}\left(-\mathbb{E}_{p(Z,Z',Y)}[Z^{\top}\widetilde{\mu}_{Y}+\|{\tau}Z'-\widetilde{\mu}_{Y}\|_{\infty}]+\mathbb{E}_{p(Z)}[\log\mathbb{E}_{p(Y)}[\exp(Z^{\top}\widetilde{\mu}_{Y})]]\right)\\
  &\quad+\frac{1}{2}\left(-\mathbb{E}_{p(Z,Z',Y)}[{Z'}^{\top}\mu_{Y}+\|{\tau}Z-\mu_{Y}\|_{\infty}]+\mathbb{E}_{p(Z')}[\log\mathbb{E}_{p(Y)}[\exp({Z'}^{\top}\mu_{Y})]]\right)\\
  &\geq \frac{1}{2}\Big(-\mathbb{E}_{p(Z,Y)}[Z^{\top}\widetilde{\mu}_{Y}]-\sqrt{\mathbb{E}_{p(Z',Y)}\|{\tau}Z'-\widetilde{\mu}_{Y}\|_{\infty}^{2}}+\mathbb{E}_{p(Z)}[\log\mathbb{E}_{p(Y)}[\exp(Z^{\top}\widetilde{\mu}_{Y})]]\Big)\\
  &\quad+\frac{1}{2}\Big(-\mathbb{E}_{p(Z',Y)}[{Z'}^{\top}\mu_{Y}]-\sqrt{\mathbb{E}_{p(Z,Y)}\|{\tau}Z-\mu_{Y}\|_{\infty}^{2}}+\mathbb{E}_{p(Z')}[\log\mathbb{E}_{p(Y)}[\exp({Z'}^{\top}\mu_{Y})]]\Big)\\
  &=\frac{1}{2}\left(-\mathbb{E}_{p(Z,Y)}[Z^{\top}\widetilde{\mu}_{Y}]-\sqrt{\mathrm{Var}(Z'|Y)}+\mathbb{E}_{p(Z)}[\log\mathbb{E}_{p(Y)}[\exp(Z^{\top}\widetilde{\mu}_{Y})]]\right)\\
  &\quad+\frac{1}{2}\left(-\mathbb{E}_{p(Z',Y)}[{Z'}^{\top}\mu_{Y}]-\sqrt{\mathrm{Var}(Z|Y)}+\mathbb{E}_{p(Z')}[\log\mathbb{E}_{p(Y)}[\exp({Z'}^{\top}\mu_{Y})]]\right)\\
  &=\frac{1}{2}\left(\mathcal{L}_{\mathrm{CE},\mathrm{Raw}}^{\widetilde{\mu}}(g_\theta)+\mathcal{L}_{\mathrm{CE},\mathrm{Aug}}^{\mu}(g_\theta)-\sqrt{\mathrm{Var}(Z'|Y)}-\sqrt{\mathrm{Var}(Z|Y)}\right)-\log{C}\\
  &=\mathcal{L}_{\mathrm{SCE}}^{\mu,\widetilde{\mu}}(g_\theta)-\frac{1}{2}\left(\sqrt{\mathrm{Var}(Z'|Y)}+\sqrt{\mathrm{Var}(Z|Y)}+2\log{C}\right).
\end{align*}
Here, in the first and the third inequality we use Jensen's inequality, and in the second inequality we use the Hölder's inequality.
\qed
\end{proof}

We next present a lower bound of the symmetric mean supervised loss.

\begin{prop}
\label{prop: lower bound for downstream tasks guarantee}
We have,
\begin{align*}
   &\;\;\;\;\mathcal{L}_{\mathrm{SCE}}^{\mu,\widetilde{\mu}}(g_\theta)-\log{C}\\
   &\geq
   -\frac{1}{2}\left(I_{\mathrm{nce}}+I_{\mathrm{nce}}'\right) - \frac{1}{2}\left(\sqrt{\mathrm{Var}(Z|Y)}+\sqrt{\mathrm{Var}(Z'|Y)}\right) - \frac{1}{2e}\mathrm{Var}(\exp({\tau}Z^{\top}Z')),
\end{align*}
where $\mathrm{Var}(\exp(\tau Z^{\top}Z'))=\mathbb{E}_{p(Z)p(Z')}[(\exp(\tau Z^{\top}Z')-\mathbb{E}_{p(Z)p(Z')}[\exp(\tau Z^{\top}Z')])^{2}]$.
\end{prop}
\begin{proof}
The proof of Proposition~\ref{prop: upper bound for downstream tasks guarantee} is mainly due to Theorem~A.5 of \citet{wang2022chaos}, but also slightly different.
Here, we show the detail of our proof based on \citet{wang2022chaos} for the sake of completeness.
 \begin{align*}
  &\frac{1}{2}\mathcal{L}_{\mathrm{SCE}}^{\mu,\widetilde{\mu}}(g_\theta)\\
  &=\frac{1}{2}\left(\mathcal{L}_{\mathrm{CE,Raw}}^{\widetilde{\mu}}(g_\theta)+\mathcal{L}_{\mathrm{CE,Aug}}^{\mu}(g_\theta)\right)\\
  &=\frac{1}{2}\mathbb{E}_{p(Z',Y)}[{Z'}^{\top}\mu_{Y}]-\frac{1}{2}\mathbb{E}_{p(Z,Y)}[{Z}^{\top}\widetilde{\mu}_{Y}] \\
  &\;\;\;\;+\frac{1}{2}\left(\mathbb{E}_{p(Z)}[\log\mathbb{E}_{p(Y)}[\exp(Z^{\top}\widetilde{\mu}_{Y})]]+\mathbb{E}_{p(Z')}[\log\mathbb{E}_{p(Y)}[\exp({Z'}^{\top}\mu_{Y})]]\right)+\log{C}\\
  &=-\frac{1}{2}\mathbb{E}_{p(Z,Z',Y)}[{\tau}{Z'}^{\top}Z+{Z'}^{\top}(\mu_{Y}-{\tau}Z)]-\frac{1}{2}\mathbb{E}_{p(Z,Z',Y)}[{\tau}{Z}^{\top}Z'+{Z}^{\top}(\widetilde{\mu}_{Y}-{\tau}Z')]\\ &\;\;\;\;+\frac{1}{2}\left(\mathbb{E}_{p(Z)}[\log\mathbb{E}_{p(Y)}[\exp(Z^{\top}\widetilde{\mu}_{Y})]]+\mathbb{E}_{p(Z')}[\log\mathbb{E}_{p(Y)}[\exp({Z'}^{\top}\mu_{Y})]]\right)+\log{C}\\
  &\geq -\frac{1}{2}\mathbb{E}_{p(Z,Z',Y)}[{\tau}{Z'}^{\top}Z+\|\mu_{Y}-{\tau}Z\|_{\infty}]-\frac{1}{2}\mathbb{E}_{p(Z,Z',Y)}[{\tau}{Z}^{\top}Z'+\|\widetilde{\mu}_{Y}-{\tau}Z'\|_{\infty}]\\
  &\;\;\;\;+\frac{1}{2}\left(\mathbb{E}_{p(Z)}[\log\mathbb{E}_{p(Y)}[\exp(Z^{\top}\widetilde{\mu}_{Y})]]+\mathbb{E}_{p(Z')}[\log\mathbb{E}_{p(Y)}[\exp({Z'}^{\top}\mu_{Y})]]\right)+\log{C}\\
  &\geq -\mathbb{E}_{p(Z,Z')}[{\tau}{Z'}^{\top}Z]-\frac{1}{2}\sqrt{\mathbb{E}_{p(Z,Y)}[\|\mu_{Y}-{\tau}Z\|_{\infty}^{2}]}-\frac{1}{2}\sqrt{\mathbb{E}_{p(Z',Y)}[\|\widetilde{\mu}_{Y}-{\tau}Z'\|_{\infty}^{2}]}\\
  &\quad+\frac{1}{2}\left(\mathbb{E}_{p(Z)}[\log\mathbb{E}_{p(Y)}[\exp(Z^{\top}\widetilde{\mu}_{Y})]]+\mathbb{E}_{p(Z')}[\log\mathbb{E}_{p(Y)}[\exp({Z'}^{\top}\mu_{Y})]]\right)+\log{C}\\
  &=-\mathbb{E}_{p(Z,Z')}[{\tau}{Z'}^{\top}Z]-\frac{1}{2}\sqrt{\mathrm{Var}(Z|Y)}-\frac{1}{2}\sqrt{\mathrm{Var}(Z'|Y)}\\
  &\;\;\;\;+\frac{1}{2}\left(\mathbb{E}_{p(Z)}[\log\mathbb{E}_{p(Y)}[\exp(Z^{\top}\widetilde{\mu}_{Y})]]+\mathbb{E}_{p(Z')}[\log\mathbb{E}_{p(Y)}[\exp({Z'}^{\top}\mu_{Y})]]\right)+\log{C}\\
  &\geq -\mathbb{E}_{p(Z,Z')}[{\tau}{Z'}^{\top}Z]-\frac{1}{2}\sqrt{\mathrm{Var}(Z|Y)}-\frac{1}{2}\sqrt{\mathrm{Var}(Z'|Y)}\\
  &\;\;\;\;+\frac{1}{2}\left(\mathbb{E}_{p(Z)}[\mathbb{E}_{p(Y)}[Z^{\top}\widetilde{\mu}_{Y}]]+\mathbb{E}_{p(Z')}[\mathbb{E}_{p(Y)}[{Z'}^{\top}\mu_{Y}]]\right)+\log{C}\\
  &=-\mathbb{E}_{p(Z,Z')}[{\tau}{Z'}^{\top}Z]-\frac{1}{2}\sqrt{\mathrm{Var}(Z|Y)}-\frac{1}{2}\sqrt{\mathrm{Var}(Z'|Y)}\\
  &\;\;\;\;+\frac{1}{2}\left(\mathbb{E}_{p(Z)}[\mathbb{E}_{p(Z')}[{\tau}Z^{\top}Z']]+\mathbb{E}_{p(Z')}[\mathbb{E}_{p(Z)}[{\tau}Z^{\top}Z']]\right)+\log{C}\\
  &= -\mathbb{E}_{p(Z,Z')}[{\tau}{Z'}^{\top}Z]-\frac{1}{2}\sqrt{\mathrm{Var}(Z|Y)}-\frac{1}{2}\sqrt{\mathrm{Var}(Z'|Y)}\\
  &\;\;\;\;+\frac{1}{2}\left(\mathbb{E}_{p(Z)}[\mathbb{E}_{p(Z')}[\log\exp({\tau}Z^{\top}Z')]]+\mathbb{E}_{p(Z')}[\mathbb{E}_{p(Z)}[\log\exp({\tau}Z^{\top}Z')]]\right)+\log{C}\\
  &\geq -\mathbb{E}_{p(Z,Z')}[{\tau}{Z'}^{\top}Z]-\frac{1}{2}\sqrt{\mathrm{Var}(Z|Y)}-\frac{1}{2}\sqrt{\mathrm{Var}(Z'|Y)}\\
  &\;\;\;\;+\frac{1}{2}\left(\mathbb{E}_{p(Z)}[\log\mathbb{E}_{p(Z')}[\exp({\tau}Z^{\top}Z')]]+\mathbb{E}_{p(Z')}[\log\mathbb{E}_{p(Z)}[\exp({\tau}Z^{\top}Z')]]\right)\\
  &\quad-\frac{1}{2e}\mathrm{Var}(\exp({\tau}Z^{\top}Z'))+\log{C}\\
  &=-\frac{1}{2}\left(I_{\mathrm{nce}}+I_{\mathrm{nce}}'\right)-\frac{1}{2}\sqrt{\mathrm{Var}(Z|Y)}-\frac{1}{2}\sqrt{\mathrm{Var}(Z'|Y)}-\frac{1}{2e}\mathrm{Var}(\exp({\tau}Z^{\top}Z'))+\log{C}.
\end{align*}
Where, in the first inequality we use Hölder's inequality, and in the second and the third inequality we apply Jensen's inequality.
In the last inequality, we utilize the sharpened Jensen's inequality \citep{liao2018sharpening}.
\qed
\end{proof}

As a direct result of Proposition~\ref{prop: upper bound for downstream tasks guarantee} and Proposition~\ref{prop: lower bound for downstream tasks guarantee}, we obtain the claim of Theorem~\ref{thm: downstream tasks guarantee analysis}.

\section{Further Comparison between Symmetric InfoNCE, InfoNCE, and SimCLR}
\label{append: Comparison between Symmetric InfoNCE, InfoNCE, and SimCLR}
Let us see additional differences between the symmetric InfoNCE and the original InfoNCE. To do so,  let us recall the following property:  due to the symmetrization, the degree of freedom of optimal critics is greatly reduced. Thus, 
from comparison between
$-(\hat{I}_{\text{nce}} + \hat{I}'_{\text{nce}}) / 2$ and $-\hat{I}_{\text{nce},q}$ of Eq.\eqref{eq:estimated infonce}, the symmetrization is expected to stabilize the parameter learning; see Figure~\ref{fig:difference between symmetric and original}.
    
For the comparison, we decompose Eq.\eqref{eq:empirical i_nce with T_s} into three terms like Eq.\eqref{eq:metric learning decomposition}. 
This decomposition can be expressed as a variant of Eq.\eqref{eq:metric learning decomposition}, %
where the last term of Eq.\eqref{eq:metric learning decomposition} is replaced by
$\frac{1}{|\mathcal{B}|}\sum_{x_i \in \mathcal{B}}\log\left(\sum_{x_j \in \mathcal{B}}e^{q\left(g_{\theta}(x_i), g_{\theta}\left(t_j(x_j)\right)\right)}\right)$. In the decomposition, following notations of Eq.\eqref{eq:metric learning decomposition}, the corresponding positive and negative losses in Eq.\eqref{eq:empirical i_nce with T_s} are denoted by $L^\prime_{\rm ps}$ and $L^\prime_{\rm ng}$, respectively.
Moreover, let us re-write
$L_{\rm ps} + L_{\rm ng}$
and
$L^\prime_{\rm ps} + L^\prime_{\rm ng}$  as
$\frac{1}{|\mathcal{B}|}\sum_{x_i \in \mathcal{B}} \ell_{\rm ps}(x_i) + \ell_{\rm ng}(x_i)$ and $\frac{1}{|\mathcal{B}|}\sum_{x_i \in \mathcal{B}} \ell^\prime_{\rm ps}(x_i) + 
\ell^\prime_{\rm ng}(x_i)$, respectively. Here, we name both $\ell_{\rm ps}(x_i) + \ell_{\rm ng}(x_i)$ and $\ell^\prime_{\rm ps}(x_i) + 
\ell^\prime_{\rm ng}(x_i)$, \textit{point-wise contrastive loss}. The four terms: $\ell_{\rm ps}$, $\ell_{\rm ng}$, $\ell^\prime_{\rm ps}$, and $\ell^\prime_{\rm ng}$ are defined as follows:
\begin{equation}
\label{eq: one term of symmetric infonce}
    \underbrace{-q\left(g_{\theta}(x_i), g_{\theta}\left(t_i(x_i)\right)\right)}_\text{$\ell_{\rm ps}(x_i)$: point-wise positive loss} + \underbrace{\frac{1}{2} \log\left(\sum_{x_j \in \mathcal{B}}\sum_{x_{i'} \in \mathcal{B}}e^{a\left(i,i',j\right)}\right)}_\text{$\ell_{\rm ng}(x_i)$: point-wise negative loss},
\end{equation}
\begin{equation}
    \label{eq:one term of infonce}
    \underbrace{-q\left(g_{\theta}(x_i), g_{\theta}\left(t_i(x_i)\right)\right)}_\text{$\ell^\prime_{\rm ps}(x_i)$: point-wise positive loss} +  \underbrace{\log\left(\sum_{x_j \in \mathcal{B}}e^{q\left(g_{\theta}(x_i), g_{\theta}\left(t_j(x_j)\right)\right)}\right)}_\text{$\ell^\prime_{\rm ng}(x_i)$: point-wise negative loss}.
\end{equation}
\begin{figure}[!t]
\centering
\includegraphics{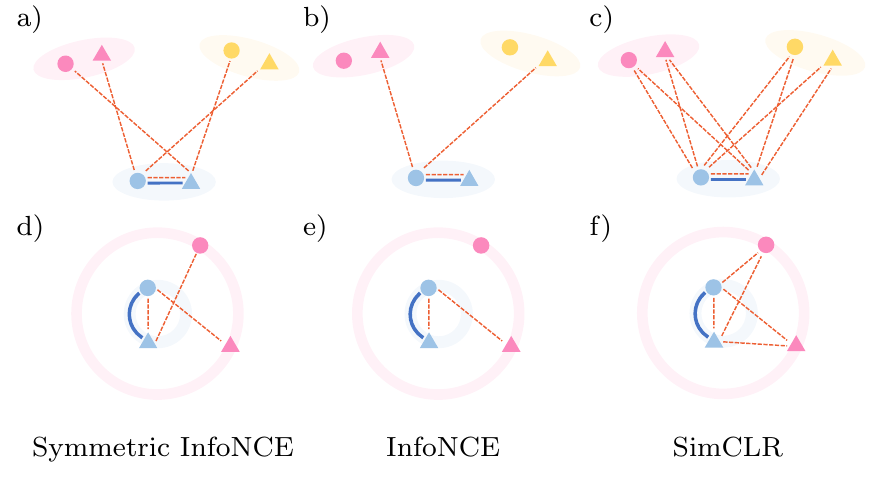}
\caption{
Illustration of positive / negative pairs of the proposed and baseline methods.
From a) to f), the colors (blue, magenta, and yellow) mean different cluster labels. The light-colored manifolds in a) and d) express true clusters. In a) (reps. d)), the set of clusters composes Three-Blobs (resp. Two-Rings).
A pair of the small circle and triangle symbols with the same color means a pair of a data point $x$ and the transformed data point $t(x)$ (i.e., so-called positive pair). 
The two data points connected by the red dash line (resp. blue straight or curved line) are enforced to be distant (resp. close) to each other. 
Especially in a) and d), the effect expressed by the red dash lines (resp. blue straight or curved line) are brought by making $\ell_{\rm ng}(x_i)$ (resp. $\ell_{\rm ps}(x_i)$) of Eq.\eqref{eq: one term of symmetric infonce} to be small. 
In b) and e), the effect expressed by the red dash lines (resp. blue straight or curved line) are brought by making $\ell^\prime_{\rm ng}(x_i)$ (resp. $\ell^\prime_{\rm ps}(x_i)$) of Eq.\eqref{eq:one term of infonce} to be small.
In c) and f), the effect expressed by the red dash lines (resp. blue straight or curved line) are brought by making $\ell''_{\rm ng}(x_i)$ (resp. $\ell''_{\rm ps}(x_i)$) of Eq.\eqref{eq: simclr pointwise contrastive loss decomposition} to be small.
}\label{fig:difference between symmetric and original}
\end{figure}
Suppose that values of the point-wise contrastive losses are small enough. In this case, we can see the  difference on stability between $-(\hat{I}_{\text{nce}} + \hat{I}'_{\text{nce}}) / 2$ and 
$-\hat{I}_{\text{nce},q}$ by comparing a) vs. b) and d) vs. e) in Figure~\ref{fig:difference between symmetric and original}. %
In this figure, it is observed that the empirical symmetric InfoNCE produces more stable contrastive effects than the empirical InfoNCE.

We also compare $-(\hat{I}_{\text{nce}} + \hat{I}'_{\text{nce}}) / 2$ and the loss of SimCLR introduced in~\citet{chen2020simple}.
To do so, consider the loss of SimCLR defined by a mini-batch $\mathcal{B}$. 
Let $L^\mathcal{B}_{\rm SimCLR}$ denote the loss of SimCLR with the mini-batch $\mathcal{B}$.
Then, let $\ell''_{\rm ps}(x_i)$ and $\ell''_{\rm ng}(x_i)$ ($x_i \in \mathcal{B}$) denote the point-wise positive loss and point-wise negative loss by,
\begin{equation}
\label{eq: simclr pointwise contrastive loss decomposition}
    L^\mathcal{B}_{\rm SimCLR} = \frac{1}{|\mathcal{B}|}\sum_{x_i \in \mathcal{B}} \ell''_{\rm ps}(x_i) + \ell''_{\rm ng}(x_i).
\end{equation}
In this case, we can see the difference via a) vs. c) and d) vs. f) in Figure~\ref{fig:difference between symmetric and original}. Since the symmetric InfoNCE produces similar contrastive effects to SimCLR does, the symmetric InfoNCE is interpreted as a simplified variant of SimCLR. We however note that it is not easy to theoretically analyze SimCLR unlike our symmetric InfoNCE, since $L^\mathcal{B}_{\rm SimCLR}$ is designed based on heuristics.

\section{Details of MIST}
\label{append:details on mist}

\subsection{Details of MIST Objective}
\label{append:details of mist objective}
To understand Eq.\eqref{eq:rewritten implementing unconstrained objective}, let us see an effect brought by minimization of each term ($R_{\rm vat}$, $-H(Y)$, $H(Y|X)$, $L_{\rm ps}$, and $L_{\rm ng}$) via Figure~\ref{fig:intuitive of our method}.
In this figure, the left pictures a) and c) show true clusters defined by the set of data points in the original space $\mathbb{R}^d$. Each color expresses a distinct true label. 
The pictures b) and d) show 
what kind of effects is brought by minimization of each term 
in the representation space $\mathbb{R}^C$. We here suppose that the appropriate hyper-parameters are used for Eq.\eqref{eq:rewritten implementing unconstrained objective}. 
In both b) and d), minimization of $R_{\rm vat}$ makes the model $g_\theta$ acquire the local smoothness; see Eq.\eqref{eq:smoothness assumption}. In addition, minimization of $L_{\rm ps}$ makes the model predict the same cluster labels for the topologically close two data points. Note that, while minimization of $L_{\rm ps}$ defined by $\mathcal{T}_{\mathfrak{e}}$ in Definition~\ref{def:process T_e} brings the similar effect (see b) in Figure~\ref{fig:intuitive of our method}) to the effect of $R_{\rm vat}$, minimization of $L_{\rm ps}$ defined by $\mathcal{T}_{\mathfrak{g}}$ in Definition~\ref{def:process T_g} brings clearly different effect from  $R_{\rm vat}$. For understanding this clear difference, observe that $x_i$ and $t_i(x_i)$ in d) are forced to be close via minimization of $L_{\rm ps}$. Minimization of $-H(Y)$ (i.e., forcing $p_\theta(y) \in \Delta^{C-1}$ to be uniform) makes the model return the non-degenerate clustering result. 
Moreover, minimization of $H(Y|X)$ makes the model return a one-hot vector. Thus, it assists each cluster to be distant. At last, 
as discussed at Section~\ref{subsec:symmetric infonce},
minimization of $L_{\rm ng}$ also makes the model  return the non-degenerate clustering result.

\begin{figure}[!t]
  \begin{algorithm}[H]
    \caption{{\bf :} MIST}
    \label{alg:proposed}
    \begin{algorithmic}[1]
     \Statex {\bf Input} Unlabeled dataset: $\mathcal{D}=\{x_i\}_{i=1}^n$. 
     Model of $p(y|x)$: $g_{\theta}(x)$. 
     Hyper-parameters: $\mu, \gamma, \eta > 0$. 
     Generative process: $\mathcal{T}$ with a conditional probability $p(t|x)$,
     where $x\in \mathbb{R}^d$ and $t: \mathbb{R}^d \to \mathbb{R}^d$. 
     If $\mathcal{D}$ is complex dataset, employ $\mathcal{T}_{\mathfrak{g}}$ of Definition~\ref{def:process T_g}. Otherwise, employ $\mathcal{T}_{\mathfrak{e}}$ of Definition~\ref{def:process T_e}.
     Mini-batch size $m$. Number of epochs: $n_{\rm ep}$. 
      \Statex \textbf{Output} Set of estimated cluster labels with $\mathcal{D}$: $\{\hat{y}_i\}_{i=1}^n$.
      \State Initialize the trainable parameter $\theta$. 
      \For{$epoch=1,\cdots,n_{\rm ep}$}
        \For{$l=0,1,\cdots,\lfloor\frac{n}{m}\rfloor$}
        \State Randomly pick up $x_{i_1},\cdots,x_{i_m}\in\mathcal{D}$. 
        \State \multiline{Compute $x_{i_k}'=t_{i_k}(x_{i_k}), i=1,\cdots,m$, where  $t_{i_k}\sim p(t|x_{i_k})$, 
        and $p(t|x_{i_k})$ is defined via either $\mathcal{T}_{\mathfrak{e}}$ or $\mathcal{T}_{\mathfrak{g}}$.}
        \State \multiline{Update the parameter $\theta$ by the SGD for the loss function in Eq.\eqref{eq:rewritten implementing unconstrained objective} 
               computed using the mini-batch $\mathcal{B}=\{x_{i_k}\}_{k=1}^m$ and $\{x_{i_k}'\}_{k=1}^m$. }
     \EndFor
     \EndFor
     \State Let $\theta^*$ be the estimated parameter. 
     \State $\hat{y}_i = \arg \max_{y \in \{1,\cdots,C\}}g_{\theta^\ast}^y(x_i)$ for $i=1,\cdots,n$. 
    \end{algorithmic}
  \end{algorithm}
\end{figure}

\begin{figure}[!t]
\centering
\def\svgwidth{\linewidth}
\includegraphics[width=\linewidth]{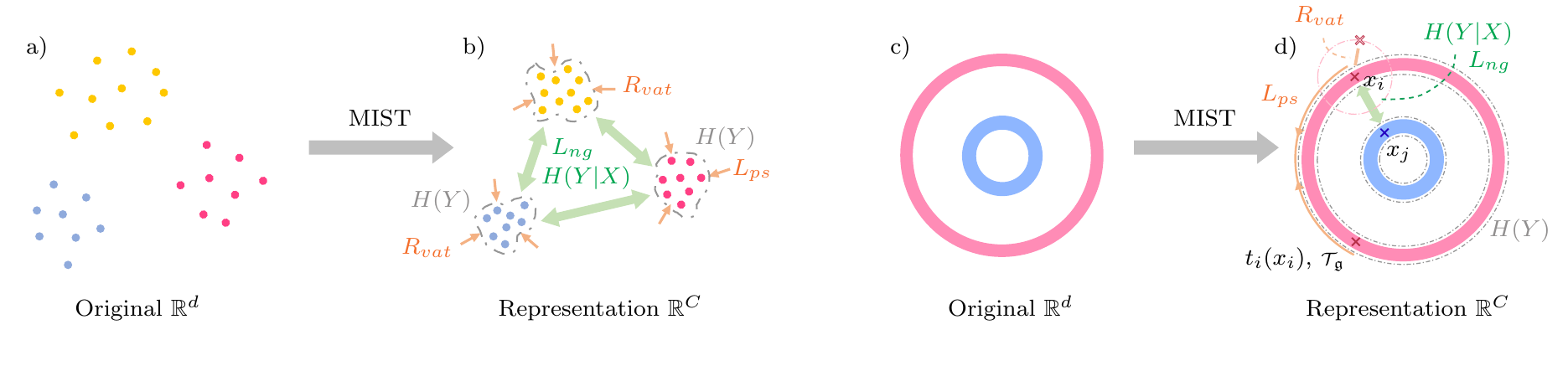}
\caption{ 
Intuitive illustration of  MIST.
Effect of each term in Eq.\eqref{eq:rewritten implementing unconstrained objective} is displayed for Three-Blobs and Two-Rings. In a) and c), the true clusters in the original space are shown, where colors mean cluster labels. In b) and d), effect of each term in Eq.\eqref{eq:rewritten implementing unconstrained objective} in the representation space is shown.
}\label{fig:intuitive of our method}
\end{figure}

\begin{figure}[!t]
\centering
\def\svgwidth{\linewidth}
\includegraphics[width=\linewidth]{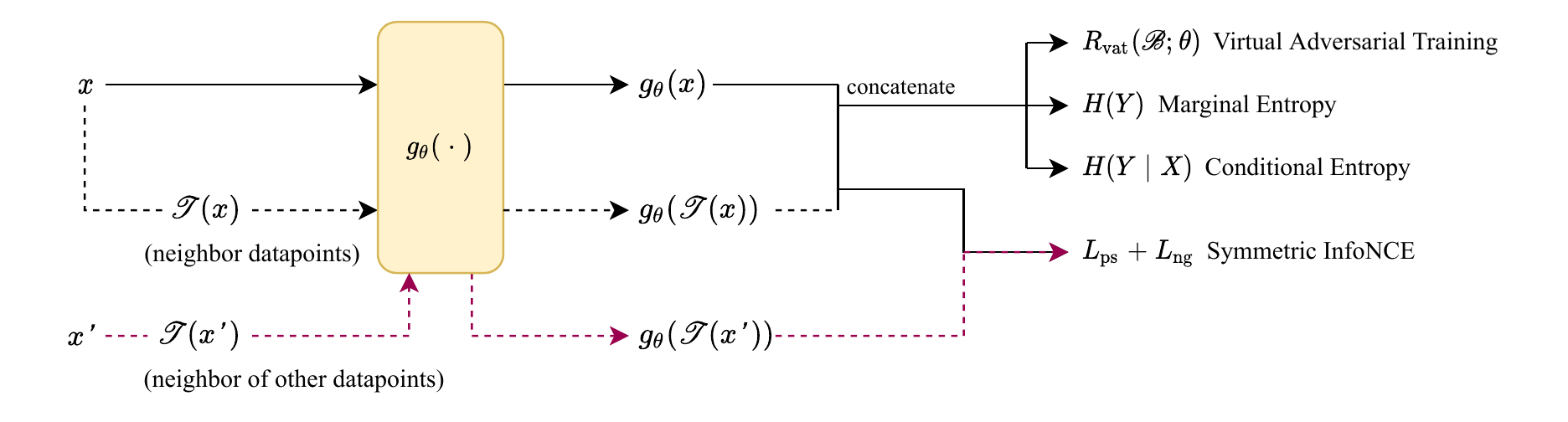}
\caption{ 
MIST architecture. 
}\label{fig:MIST diagram}
\end{figure}

As recently proposed methods that are similar to MIST, we list~\citet{van2020scan,li2020contrastive,Dang_2021_CVPR}. 
The above three methods focus on the image domain (i.e., \textsf{Scenario1} of Section~\ref{subsec:our scenario}).
In the above three methods, either InfoNCE or SimCLR is employed
to enhance the clustering performance. 
Moreover, the scenario these related works focus on is different from \textsf{Scenario2}.
Furthermore, all the three studies do not provide any theoretical analysis of their proposed methods.

\subsection{Time and Memory Complexities with $\mathcal{T}_{\mathfrak{e}}$ and $\mathcal{T}_{\mathfrak{g}}$}
\label{append:time and memory complex with t_e and t_g}
Suppose that we construct the non-approximated K-NN graph on $\mathcal{D}$ by using Euclidean distance. 
Then, the time complexity with $\mathcal{T}_{\mathfrak{e}}$  is $O(dn^2)$, 
where $d$ is the dimension of a feature vector. 
The memory complexity is $O(K_0 n)$. 
As for 
$\mathcal{T}_{\mathfrak{g}}$, time and memory complexities are  
$O\left((K_0 + \log n)n^2\right)$ and $O(n^2)$, respectively~\citep{moscovich2017minimax}. 
Note that if we construct the approximated K-NN graph on $\mathcal{D}$ 
by the Euclidean distance, the time complexity with $\mathcal{T}_{\mathfrak{e}}$ is reduced to $O(dn\log n)$~\citep{Wang2013FastAF,zhang2013fast}.

\section{Experiment Details}
\label{append: experimental details}

\subsection{Details of Datasets}
\label{append:Detail of Dataset}
We used Two-Moons\footnote{\url{https://scikit-learn.org/stable/modules/generated/sklearn.datasets.make\_moons.html} [Last accessed 23-July-2022]} and Two-Rings\footnote{\url{https://scikit-learn.org/stable/modules/generated/sklearn.datasets.make\_circles.html} [Last accessed 23-July-2022]}
in scikit-learn.  
For the former dataset, we set $0.05$ as the noise parameter. For the latter dataset
we set $0.01$ and $0.35$ as noise and factor parameters respectively. %
For SVHN, STL, CIFAR10, CIFAR100, Omniglot and Reuters10K, 
we used the datasets on GitHub\footnote{\url{https://github.com/weihua916/imsat} [Last accessed 23-July-2022]}. 
As for MNIST and 20news, Keras~\citep{geron2019hands} was used. The summary of all the datasets is shown in Table~\ref{tb:dataset statistics}. 
In the following, we review how features of the eight real-world datasets are obtained.

\begin{table*}[!t]
  \centering
   \caption{Summary of the ten datasets used in our experiments. 
   }
   \label{tb:dataset statistics}
      \scalebox{0.8}{
      \begin{tabular}{lccccc} 
      \toprule
                   & \#Points & \#Cluster & Dimension & \%Largest Cluster & \%Smallest Cluster \\ \midrule%\hline \hline
        Two-Moons  & 5000  & 2  & 2    & 50\% & 50\% \\ 
        Two-Rings  & 5000  & 2  & 2    & 50\% & 50\% \\ 
        \midrule
        MNIST      & 70000 & 10 & 784  & 11\% & 9\% \\ 
        SVHN       & 99289 & 10 & 960  & 19\% & 6\% \\ 
        STL        & 13000 & 10 & 2048 & 10\% & 10\% \\
        CIFAR10    & 60000 & 10 & 2048 & 10\% & 10\% \\
        CIFAR100   & 60000 & 100 & 2048 & 1\% & 1\% \\ 
        Omniglot   & 40000 & 100 & 441 & 1\% & 1\% \\ 
        20news     & 18040 & 20 & 2000 & 5\% & 3\% \\ 
        Reuters10K & 10000 & 4 & 2000 & 43\% & 8\% \\ 
      \bottomrule
      \end{tabular}
                    }
\end{table*}

\begin{itemize}
    \item MNIST: It is a hand-written digits classification dataset with $28\times28$ single-channel images.
    The value of each pixel is linearly normalized into $\left[0,1\right]$ and then flattened to a $784$ dimensional feature vector.

     \item STL: It is a labelled subset of ImageNet~\citep{deng2009imagenet} with $96\times96$ colored images. We adopted features from~\citet{hu2017}, which is extracted by pre-trained 50-layer ResNets.

    \item CIFAR10: It is a dataset with ten clusters, $32\times32$ colored images. We adopted features from~\citet{hu2017}, which is extracted by pre-trained 50-layer ResNets.
    
    \item CIFAR100: It is a dataset with one hundred clusters, $32\times32$ colored images. We adopted features from~\citet{hu2017}, which is extracted by pretrained 50-layer ResNets.
    
    \item Omniglot: It is a hand-written character recognition dataset. We adopted the processing results from~\citet{hu2017}, which is an one hundred clusters dataset with twenty unique data points per class. Twenty times affine augmentations were applied as in~\citet{hu2017}, so there are $100\times20\times20=40000$ images available. Images were sized $21\times21$ single-channel, linearly normalized into $\left[0,1\right]$ and flattened into feature vectors.
    
    \item 20news: It is a dataset of news documents across twenty newsgroups. We adopted the processing code from~\citet{hu2017}. It used the data from
    python package scikit-learn~\citep{geron2019hands} and processed using tf-idf
    features with 'english' stopwords.
    
    \item SVHN: It is a dataset with street view house numbers. Following~\citet{hu2017}, we used the features they have extracted with $960$-dimensional GIST
    feature~\citep{oliva2001modeling}.

    \item Reuters10K: It is a dataset with English news stories. We adopted
    the processing results from~\citet{hu2017}. It contains four categories as labels: corporate/industrial, government/social, markets and economics.
    ten-thousands documents were randomly sampled, and processed without stop words. tf-idf features were used as in~\citet{hu2017}.
\end{itemize}

\subsection{Complex and Non-Complex Topology Datasets}
\label{append: Complex and Non-Complex Topology Datasets}

In order to characterize each dataset from some geometric point of view, we performed experiments with the K-means algorithm for these ten datasets; see the top row of Table~\ref{tb:results of clustering accuracy}.
Here, in the K-means algorithm we use the Euclidean distance to measure how far two points are apart from each other: see Chapter~22 of \citet{shalev2014understanding} for a general objective function of the K-means algorithm.
Hence, if the Top-1 accuracy with the K-means algorithm is low, then the dataset can have a complex structure so that the K-means algorithm fails to group the data points into meaningful clusters.
Utilizing the results with K-means algorithm (see the second row of Table~\ref{tb:results of clustering accuracy}), we define (non-)complex topology of a dataset as follows: 1) we say a dataset has non-complex topology if the Top-1 accuracy (\%
According to these definitions, we classify the ten datasets into two categories; two synthetic datasets, Two-Moons and Two-Rings, are of complex topology, and the others are of non-complex topology.

Note that, strictly speaking, it is difficult to provide a rigorous definition of (non-)complex topology for a real-world dataset. Instead, we state a definition inspired by our empirical observations with the K-means algorithm for the ten different datasets.

\subsection{Implementation Details with Compared Methods}
\label{append:Implementation Details with Previous Methods}

\begin{itemize}
    \item K-means: sklearn.cluster.KMeans from scikit-learn.
    \item SC: sklearn.cluster.SpectralClustering from scikit-learn with 50 - 'nearest neighbors' Graph and 'amg' eigen solver.
    \item GMMC: sklearn.mixture.GaussianMixture from scikit-learn with diagonal covariance matrices.

    \item DEC\footnote{\url{https://github.com/XifengGuo/DEC-keras} [Last accessed 23-July-2022]}: 
    Keras implementation of \citet{xie2016unsupervised} is used. 

    \item SpectralNet\footnote{\url{https://github.com/KlugerLab/SpectralNet} [Last accessed 23-July-2022]}: 
    We used the version at commit \emph{ce99307} with tensorflow 1.15, keras 2.1.6, Ubuntu 18.04 since we found that this is the only configuration that reproduces paper result in our environments. For real-world datasets, we used the 10-dimensional VaDE representation obtained in this work (see implementation details of VaDE) as input to SpectralNet. 10 neighbors were used with approximated nearest neighbor search. For Toy-sets, we have used the raw 2-dimensional input with official hyper-parameter setups for "CC" dataset in SpectralNet.

    \item VaDE\footnote{\url{https://github.com/GuHongyang/VaDE-pytorch} [Last accessed 23-July-2022]}: 
    We added the constraint that Gaussian Mixture component weight $\pi > 0$ to avoid numerical instabilities. We did not use the provided pretraining weights since we cannot reproduce the pretraining process for all datasets.

    \item IMSAT\footnote{\url{https://github.com/betairylia/IMSAT\_torch} [Last accessed 23-July-2022]}: Given an unlabeled dataset $\mathcal{D}$ and the number of clusters, we train a clustering model of IMSAT by using $\mathcal{D}$ via Eq.\eqref{eq:practical imsat objective}. In addition, we define the adaptive radius $\epsilon_i$ in VAT as same with one defined in MIST: see also Appendix~\ref{subsec:Hyperparameter Setting with Proposed Method}. Moreover, for synthetic datasets, we set $(0.1, 0.5)$ to $(\lambda_1, \lambda_2)$ of Eq.\eqref{eq:practical imsat objective}, and set $0.1$ to $\xi$ in VAT.  For real-world datasets, we set $(0.1, 4)$ to $(\lambda_1, \lambda_2)$ of Eq.\eqref{eq:practical imsat objective}, and set $10$ to $\xi$ in VAT. 

    \item IIC\footnote{\url{https://github.com/betairylia/MIST} [Last accessed 23-July-2022]}: Since we consider \textsf{Scenario2}, we cannot define the transformation function via the domain-specific knowledge. Therefore, we define it via $\mathcal{T}_{\mathfrak{e}}$ of Definition~\ref{def:process T_e} for all ten datasets as follows. For the synthetic datasets and the image datasets, $K_0 = 10$ is used. For the text datasets, $K_0 = 100$ is used. The above values of $K_0$ are selected by the hyper-parameter tuning.   %

    \item 
    CatGAN: We adopted the implementation from here\footnote{\url{https://github.com/xinario/catgan\_pytorch} [Last accessed 23-July-2022]} and moved it to GPU. 
    Since the original CatGAN experiments have used CNNs and cannot be applied to general-purpose datasets, we substituted CNNs in both generator and discriminator with a 4-layer MLP.

    \item SELA: 
    We used the official implementation\footnote{\url{https://github.com/yukimasano/self-label} [Last accessed 23-July-2022]} 
    with single head and known cluster numbers. We replaced the Convolutional Network in the original work by a simple MLP identical to our MIST implementation as we focusing on general purpose unsupervised learning instead of images. We also disabled data-augmentation steps presented in the original work of SELA.

    \item SCAN: 
    We adopted the loss computation part from official implementation\footnote{\url{https://github.com/wvangansbeke/Unsupervised-Classification} [Last accessed 19-July-2022]} and used MIST's framework to implement SCAN. Since we focus on generic datasets without specific domain knowledge, data augmentations are removed and SCAN learns solely on nearest neighbors. Same input data (and feature extraction steps) as MIST are used for our SCAN implementation.

\end{itemize}

\subsection{Two-Dimensional Visualization}
\label{append:Visualization Details in Fig of imsat and spectralnet weakness}
Panels a)$\sim$h) in Figure~\ref{fig:weakness of imsat and spectralnet} were obtained by the following procedure. For two-dimensional visualization with real-world datasets, we employ UMAP~\citep{mcinnes2018umap}, 
and implement it using the public code\footnote{\url{https://pypi.org/project/umap-learn/} [Last accessed 23-July-2022]}, where we set ten and two to "n$\_$neighbors" and "n$\_$components". In addition, we fix the above two parameters with UMAP for all visualization of real-world datasets.

    \begin{figure}[!t]
    \centering
    \def\svgwidth{\linewidth}
    \includegraphics[scale=0.9]{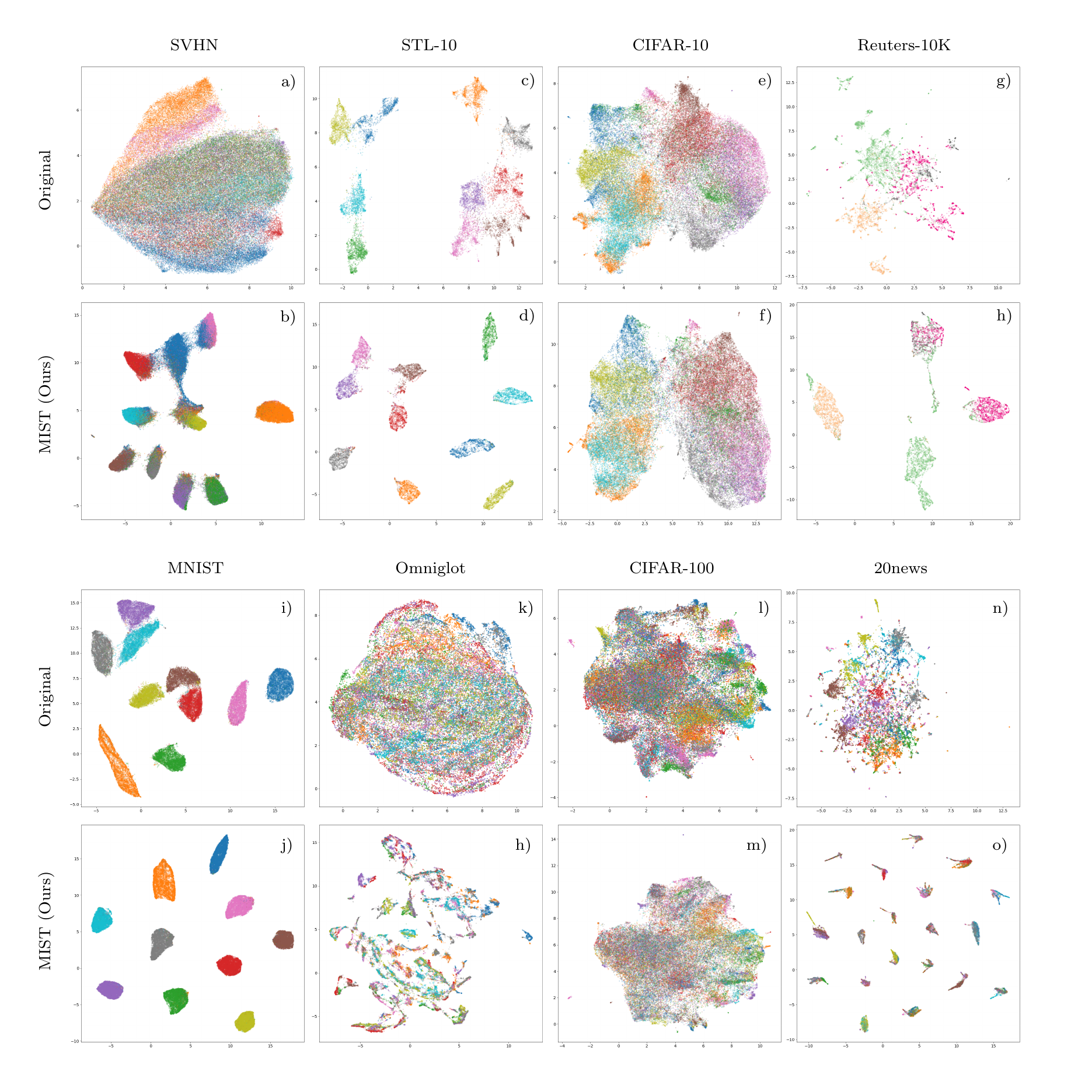}
    \caption{Two-dimensional visualizations of original datasets and their representations by MIST. Visualization of a original dataset is obtained via the same manner with panel a) of Figure~\ref{fig:weakness of imsat and spectralnet}, while that of MIST representation is obtained via the same manner with panel g) of Figure~\ref{fig:weakness of imsat and spectralnet}.
    }
    \label{fig:umap visualization svhn, stl, cifar10 and r10k}
    \end{figure}

\begin{itemize}
    \item a): Input MNIST dataset $\mathcal{D}_{\rm mnist} =\{x_i\}_{i=1}^n$, where $x_i \in \mathbb{R}^{784}$ and $n=70000$, to UMAP. Then, we obtain the two-dimensional vectors of $\mathcal{D}_{\rm mnist}$. Then, assign true labels to the vectors.

    \item c): Firstly, using $\mathcal{D}_{\rm mnist}$, train IMSAT of Eq.\eqref{eq:practical imsat objective} where $(\lambda_1, \lambda_2)=(0.1, 4)$. Moreover, for VAT in IMSAT, we set ten to $\xi$, and define the adaptive radius $\epsilon_i$ as same with one defined in MIST; see also Appendix~\ref{subsec:Hyperparameter Setting with Proposed Method}. Input $\mathcal{D}_{\rm mnist}$ to the trained clustering model whose last layer (a softmax function) is removed.
    Then, get the output whose dimension is $C=10$. Thereafter, we feed the output to UMAP, and we obtain the two-dimensional vectors. Thereafter, assign true labels to the vectors.

    \item e): 
    Firstly, using $\mathcal{D}_{\rm mnist}$, train a clustering MLP of SpectralNet. SpectralNet's official hyper-parameter setups are used.
    Input $\mathcal{D}_{\rm mnist}$ to the trained clustering model whose last layer (a softmax function) is removed.
    Then, get the output whose dimension is $C=10$. Thereafter, we feed the output to UMAP, and we obtain the two-dimensional vectors. Thereafter, assign true labels to the vectors.

    \item g): Firstly, using $\mathcal{D}_{\rm mnist}$, train clustering neural network model with MIST whose hyper-parameters are defined in Appendix~\ref{subsec:Hyperparameter Setting with Proposed Method}. Input $\mathcal{D}_{\rm mnist}$ to the trained clustering model whose last layer defined by the softmax is removed, and get the output. Then, we input the output to UMAP, and we obtain the two-dimensional vectors. Thereafter, assign true labels to the vectors.
    
    \item b), d), f), h): Since a data point in Two-Rings dataset $\mathcal{D}_{\rm two\_rings}$ already belongs to two-dimensional space, we just visualize the data point location with its label information in the panel b). With detail of $\mathcal{D}_{\rm two\_rings}$, see Appendix~\ref{append:Detail of Dataset}.
    For the panels d), f) and h), we firstly predict the cluster labels by using corresponding clustering method.
    Then, visualize the data point location with its predicted cluster label.   
    
\end{itemize}

In Figure~\ref{fig:umap visualization svhn, stl, cifar10 and r10k},
we additionally show two-dimensional visualization results 
of all eight real-world datasets.
In this figure, visualizations of the first row were obtained by the same manner with the panel a) of Figure~\ref{fig:weakness of imsat and spectralnet}. Visualizations (by MIST) of the second row in Figure~\ref{fig:umap visualization svhn, stl, cifar10 and r10k} were obtained by the same manner with the panel g) of Figure~\ref{fig:weakness of imsat and spectralnet}.

\subsection{Hyper-Parameter Tuning}
\label{subsec:Hyperparameter Setting with Proposed Method}

\begin{table}[!t]
  \centering
   \caption{All hyper-parameters related to MIST of Algorithm~\ref{alg:proposed}. In the first (resp. second) column, the hyper-parameters (resp. reference) are shown. 
   }
   \label{tb:summary of hyperparameters}
   \scalebox{0.85}{
      \begin{tabular}{cc} 
      \toprule
      Hyper-Parameters         &  Reference       \\ \midrule%
      $(\mu, \eta, \gamma)$    & MIST objective of Eq.\eqref{eq:rewritten implementing unconstrained objective} \\ 
      $(\alpha, \tau)$       &    Critic function of Eq.\eqref{eqn:alpha-tau-model}      \\ %
      $(K_0, \beta)$           & Definition~\ref{def:process T_e} and~\ref{def:process T_g}\\  
      $(\xi, \epsilon_i)$    &   VAT of Eq.\eqref{eq: vat-loss} \\ \midrule
      $(m, n_{\rm ep})$ & Algorithm~\ref{alg:proposed} \\
      \bottomrule
      \end{tabular}
      }
\end{table}

Table~\ref{tb:summary of hyperparameters} shows all hyper-parameters related to MIST algorithm. 
Throughout numerical experiments of Section~\ref{sec:numerical-experiments}, we set $250, 50$ as $m,n_{\rm ep}$, respectively. In addition, following~\citet{hu2017}, we respectively fix $\epsilon_i$ of VAT to $\epsilon_i = 0.25 \times \left\|x_i - x^{(10)}_i\right\|_2$, where $x_i \in \mathcal{D}$ and $x^{(10)}_i$ is the tenth nearest neighbor data point from $x_i$ on $\mathcal{D}$ with the Euclidean metric. Note that for the synthetic datasets (resp. real-world datasets), the generative process $\mathcal{T}_{\mathfrak{g}}$ of Definition~\ref{def:process T_g} (resp. the generative process $\mathcal{T}_{\mathfrak{e}}$ of Definition~\ref{def:process T_e}) is employed.

\begin{table}[!t]
  \centering
   \caption{Candidates with some hyper-parameters related in MIST and MIST with $\hat{I}_{\rm nce}$ of Table~\ref{tb:results of clustering accuracy}.
   The symbol of $\leftarrow$ indicates the same value to the cell in the left. Real-world \& Image means MNIST, SVHN, STL, CIFAR10, CIFAR100, Omniglot. Real-world \& Text means 20news and Reuters10K. Synthetic means both Two-Moons and Two-Rings. 
   }
   \label{tb:candidates of hyperparameters}
   \scalebox{0.8}{
      \begin{tabular}{cccc} 
      \toprule
      Hyper-Parameters         &  Real-world \& Image & Real-world \& Text    &  Synthetic   \\ \midrule%
      \multirow{2}{*}{$(\mu, \eta, \gamma)$}    &  $\{(0.045, 5, 1.5), (0.045, 6, 1.5),$ & \multirow{2}{*}{$\leftarrow$}  & $\{(0.1, 15, 10), (0.1, 15.5, 10),$\\ 
      &$\;\;\;(0.05, 5, 1.5), (0.04, 6, 1.5)\}$ &&$(0.1, 16, 10)\}$ \\
      $\alpha$                 &    $\{0,1,2\}$        & $\leftarrow$ & $\leftarrow$   \\ %
      $\tau$                   &    $\{0.01,0.05,0.1,1,10\}$            & $\leftarrow$   &        $\leftarrow$    \\
      $(K_0, \beta)$           &    $\{(5, 0), (7,0), (10,0), (15,0)\}$ &  $\{(50j,2/3), (50j,4/5)\;|\;j\in\{1,..,4\}\}$     &  $\{(15, j/10)\;|\;j\in\{0,..,10\}\}$    \\  
      $\xi$                    &   $\{0.1, 1, 10, 100\}$      & $\leftarrow$ & $\leftarrow$ \\ 
      \bottomrule
      \end{tabular}
      }
\end{table}

In numerical experiments of Table~\ref{tb:results of clustering accuracy}, the other hyper-parameters are tuned within the corresponding candidates shown in Table~\ref{tb:candidates of hyperparameters}. Those candidates were decided by the following procedure:
\begin{itemize}
    \item $(\mu, \eta, \gamma)$: Since MIST is based on IMSAT, following~\citet{hu2017}, for real-world datasets, we manually search efficient candidates, which safisfy the following criterion, inside the region including $\mu\eta = 0.4$ and $\mu=0.1$: candidates which work well for MIST and MIST with $\hat{I}_{\rm nce}$ of Table~\ref{tb:results of clustering accuracy}.
    Note that, in IMSAT objective of Eq.\eqref{eq:practical imsat objective}, the authors set $\mu=0.1$ and $\eta=4$ in their official code. For the synthetic datasets, the candidates were decided via totally manual searching.
    
    \item $(\alpha, \tau)$ and $(K_0, \beta)$: We essentially conducted manual searching for candidates, which can be efficient for both MIST and MIST with $\hat{I}_{\rm nce}$. When we select the candidates of $K_0$, 
    we follow the same strategy of~\citet{shaham2018}. 
    
    \item $\xi$: 
    We chose values that are around ten, since ten is set as $\xi$ in the official IMSAT code.
\end{itemize}

As for criterion of hyper-parameter tuning of the MIST and MIST with $\hat{I}_{\rm nce}$, we employed the following: for each (either real-world or synthetic), the most efficient $(\mu, \eta, \gamma, \alpha, \tau, \xi)$ should be found, while $(K_0, \beta)$ can be adaptive for ten datasets. To find the best efficient one, we used a tuning method described in Appendix G of~\citet{hu2017}, where a set of hyper-parameters, that can have the highest average clustering accuracy over several datasets, are selected. The tuning result of the MIST is shown in Table~\ref{tb:selected hyperparameters}.

\begin{table}[!t]
  \centering
   \caption{Selected hyper-parameters with MIST of Table~\ref{tb:results of clustering accuracy}. The symbol of $\leftarrow$ indicates the same value to the cell in the left. 
   }
   \label{tb:selected hyperparameters}
   \scalebox{0.8}{
      \begin{tabular}{ccccccc} 
      \toprule
                               &  MNIST, SVHN, STL, CIFAR100, Omniglot  &  CIFAR10  &    20news    &  Reuters10K  &  Two-Moons & Two-Rings        \\ \midrule%
      $(\mu, \eta, \gamma)$    &    { $(0.045, 5, 1.5)$}      &$\leftarrow$&$\leftarrow$&$\leftarrow$& $(0.1, 15.5, 10)$  & $\leftarrow$ \\ %
      $(\alpha,\tau)$                 &    $(1,0.05)$                       &$\leftarrow$&$\leftarrow$&$\leftarrow$&$\leftarrow$ & $\leftarrow$       \\ %
      $(K_0, \beta)$           &  $(7,0)$                        & $(15,0)$  & $(200, 4/5)$ & $(50, 4/5)$  &      $(15, 0)$  & $(15, 0.6)$\\ %
      $\xi$     & 10 & $\leftarrow$ & $\leftarrow$ & $\leftarrow$ & 0.1 & $\leftarrow$\\
      \bottomrule
      \end{tabular}
      }
\end{table}

Moreover, for all combinations except for (\ctext{B}, \ctext{C}) in Table~\ref{table: ablation study result}, we at first manually selected the candidates of hyper-parameters. Then, for each combination, we conducted hyper-parameter tuning, whose criterion is same with one employed for tuning hyper-parameters in MIST of Table~\ref{tb:results of clustering accuracy}. The tuning results are shown in Table~\ref{tb:selected hyperparameters for ABCD ablation}.

\begin{table}[!t]
  \centering
   \caption{
   Selected value for each hyper-parameter with experiments related to Table~\ref{table: ablation study result}.
   From the second to sixth columns, hyper-parameter values selected for real-world datasets are shown. The symbol of "$-$" means that 
   the hyper-parameter is not needed. 
   The symbol of $\leftarrow$ indicates the same value to the cell in the left. In the last column, hyper-parameter values to define six combinations are shown. As for $(\alpha, K_0, \beta, \xi)$, the same values shown in Table~\ref{tb:selected hyperparameters} are employed.
   }
   \label{tb:selected hyperparameters for ABCD ablation}
   \scalebox{1}{
      \begin{tabular}{ccccccc} 
      \toprule
      &  (\ctext{D})  & (\ctext{B}, \ctext{D})  &  (\ctext{A}, \ctext{D})  &  (\ctext{B}, \ctext{C}, \ctext{D}) & (\ctext{A}, \ctext{B}, \ctext{C}) &  Two-Rings  \\ \midrule%
      $\mu$    & $-$ &  $\leftarrow$ & 0.045 & $-$ & 0.1 & 0.1 \\ %
      $\eta$   & $-$  & 1 & $-$ & 1 & 4 & 15.5 \\ %
      $\gamma$ & $-$  & 10 & 1.5 & 10 & $-$ & 10 \\
      $\tau$   & 1.0  & $\leftarrow$ & $\leftarrow$ & 0.1 & $-$ & 1.0 \\%
      \bottomrule
      \end{tabular}
      }
\end{table}

\clearpage

\bibliographystyle{apalike}
%\bibliography{allref_neco}

\end{document}